\documentclass[journal]{IEEEtran}
\usepackage{ifpdf}

\usepackage{cite}

\ifCLASSINFOpdf
  \usepackage[pdftex]{graphicx}
  \graphicspath{{../pdf/}{../jpeg/}}
  \DeclareGraphicsExtensions{.pdf,.jpeg,.png}
\else
  \usepackage[dvips]{graphicx}
  \graphicspath{{../eps/}}
  \DeclareGraphicsExtensions{.eps}
\fi
\usepackage{amssymb}
\usepackage[cmex10]{amsmath}
\usepackage{array}
\usepackage{fixltx2e}
\usepackage{dblfloatfix}

\usepackage{amsfonts}
\makeatletter
\newcommand{\eqnum}{\leavevmode\hfill\refstepcounter{equation}\textup{\tagform@{\theequation}}}
\makeatother

\usepackage{amsthm}
\usepackage{slashbox}
\usepackage[all]{xy}
\usepackage{algpseudocode}
\usepackage{varwidth}
\usepackage{calc}
\usepackage{framed}
\usepackage{empheq}

\newcommand{\mb}[1]{\boldsymbol{\mathbf{#1}}}
\newcommand{\EA}[1]{\textnormal{EA}_{#1}}
\newcommand{\PP}[3][P]{\mb{#1}_{#2}^{#3}}
\newcommand{\Pkn}{\PP{k}{n}}
\newcommand{\Qkn}{\PP[Q]{k}{n}}
\newcommand{\Pkinf}{\PP{k}{\infty}}
\newcommand{\Qkinf}{\PP[Q]{k}{\infty}}
\newcommand{\Xkin}[4][x]{\mb{#1}_{{#2},{#3}}^{#4}}
\newcommand{\Xkn}[3][x]{\mb{#1}_{#2}^{#3}}

\newcommand{\EXP}{\operatorname{E}}
\newcommand{\VAR}{\operatorname{V}}
\newcommand{\COV}{\operatorname{C}}

\newcommand{\xrarrow}[1]{\xrightarrow{\textnormal{#1}}}
\newcommand{\tod}{\xrarrow{d}}
\newcommand{\tompw}{\xrarrow{m.p.w.}}
\newcommand{\toas}{\xrarrow{a.s.}}
\newcommand{\totv}{\xrarrow{tv}}
\newcommand{\dist}[1]{\rho_{\textnormal{#1}}}
\newcommand{\disttv}{\dist{tv}}
\newcommand{\distd}{\dist{d}}
\newcommand{\identity}{\textnormal{I}}

\newcommand{\Prob}{\textnormal{P}}

\newcommand{\UI}{\mathbb{U}_\textnormal{I}}

\newcommand{\Law}{\mathcal{L}}

\newcommand{\sigmafield}{}
\def\sigmafield/{$\sigma$-field}

\newtheorem{theorem}{Theorem}

\newtheorem{lemma}{Lemma}
\newtheorem{definition}{Definition}

\begin{document}
%
\title{A Revisit of Infinite Population Models for Evolutionary Algorithms on Continuous Optimization Problems}
%
%
%

\author{Bo~Song
        and~Victor~O.K.~Li,~\IEEEmembership{Fellow,~IEEE}
\thanks{Bo Song and Victor O.K. Li are with the Department of Electrical and Electronic Engineering, the University of Hong Kong, Pokfulam, Hong Kong (e-mail: bsong@connect.hku.hk; vli@eee.hku.hk).}
}
%
%

\markboth{IEEE TRANSACTIONS ON EVOLUTIONARY COMPUTATION}%
{Song \MakeLowercase{\textit{et al.}}: A Revisit of Infinite Population Models for Evolutionary Algorithms on Continuous Optimization Problems}
%



\maketitle

\begin{abstract}
Infinite population models are important tools for studying population dynamics of evolutionary algorithms. They describe how the distributions of populations change between consecutive generations. In general, infinite population models are derived from Markov chains by exploiting symmetries between individuals in the population and analyzing the limit as the population size goes to infinity. In this paper, we study the theoretical foundations of infinite population models of evolutionary algorithms on continuous optimization problems. First, we show that the convergence proofs in a widely cited study were in fact problematic and incomplete. We further show that the modeling assumption of exchangeability of individuals cannot yield the transition equation. Then, in order to analyze infinite population models, we build an analytical framework based on convergence in distribution of random elements which take values in the metric space of infinite sequences. The framework is concise and mathematically rigorous. It also provides an infrastructure for studying the convergence of the stacking of operators and of iterating the algorithm which previous studies failed to address. Finally, we use the framework to prove the convergence of infinite population models for the mutation operator and the $k$-ary recombination operator. We show that these operators can provide accurate predictions for real population dynamics as the population size goes to infinity, provided that the initial population is identically and independently distributed.
\end{abstract}


%
\IEEEpeerreviewmaketitle

\section{Introduction}
%
%
%
%

\label{intro}

\IEEEPARstart{E}{volutionary} algorithms (EAs) are general purpose optimization algorithms which saw great successes in real-world applications. They are inspired by the evolutionary process in nature. A certain number of candidate solutions to the problem at hand are modeled as individuals in a population, and through generations the algorithm evolves the population by producing new individuals and selectively replacing the old ones. The idea is that the survival probabilities of individuals in the population are related to their objective function values, or fitness values in this context. In general, individuals with more preferable objective function values or higher fitness values are more likely to survive and remain in the next generation. As a result, by the ``survival of the fittest'' principle, it is likely that after many generations the population will contain individuals with sufficiently high fitness values, such that these individuals are satisfactory solutions to the problem at hand.

Though conceptually simple, the underlying evolutionary processes and the behaviors of EAs remain to be fully understood. The difficulties lie in the fact that EAs are customizable population-based iterative stochastic algorithms, and the objective function also has great influence on their behaviors. A successful model of EAs should account for both the mechanisms of the algorithm and the influence from the objective function. One way to derive such models is to study EAs as dynamical systems. The idea is to pick a certain quantity of interest first, such as the distribution of the population or a certain statistic about it. Then, transitions in the state space of all possible outcomes about the picked quantity are studied. A Markov chain described by a transition matrix (when the state space is finite) or a difference equation (when the state space is not finite) is derived to describe how the picked quantity changes between consecutive generations.

Although dynamical system approach brings many insights about EAs, the state spaces of the models tend to grow rapidly as the population size increases. This is because in order to characterize the population dynamics accurately, the state space in the model has to be large enough to describe all the interdependencies between individuals in the current and next generations. As a result, even for time-homogeneous EAs with moderate population size, the dynamical system model is often too large and too complex to be analyzed or simulated. To overcome this issue, some researchers instead turn to studying the limiting behaviors of EAs as the population size goes to infinity. The idea is to exploit some kind of symmetry in the state space (such as all individuals have the same marginal distribution), and prove that in the limit the Markov chain can be described by a more compact model. The models built in this way are called infinite population models (IPMs).


In this paper, we follow this line of research and study IPMs of EAs on \emph{continuous} space. More specifically, we aim at rigorously proving the convergence of IPMs. Notice that in this study by convergence we usually mean a certain property of IPMs. That an IPM converges loosely means that as the population size goes to infinity, the population dynamics of the real EA converge in a sense to the population dynamics predicted by this model. This usage is different from conventional ones where it means that the EA eventually locates and gets stuck in some local or global optima. Convergence results guarantee that IPMs characterize some kind of limiting behaviors of real EAs. They are the foundations and justifications of IPMs. 

To our knowledge, there are very few research efforts which directly studied the convergence of IPMs. Among them, the studies of Qi et al. in\cite{Qi1,Qi2} are the classic and most relevant ones. Qi et al. studied the population dynamics of simple EA on continuous space. In the first part of their research\cite{Qi1}, the authors built an IPM to analyze the population dynamics of simple EA with proportionate selection and mutation. Traditionally, a transition equation is constructed to describe how the probability density functions (p.d.f.s) of the \emph{joint} distributions of individuals change between consecutive generations. The novelty of the authors' research lies in their introduction of the modeling assumption that individuals in the same generation are exchangeable, and therefore they all have the same \emph{marginal} distribution. Then, as a key result, the authors proved that as the population size goes to infinity, the marginal p.d.f.s of the populations produced by the real algorithm converge \emph{point-wisely} to the p.d.f.s. predicted by the following transition equation:
\begin{equation} \label{eqn:intro-conve}
	f_{\mb{x}_{k+1}}(x) = \frac{\int_{\mathbb{F}} \! f_{\mb{x}_{k}}(y) g(y) f_{w}(x|y)  \mathrm{d}y}{\int_{\mathbb{F}} \! f_{\mb{x}_{k}}(y)g(y)  \mathrm{d}y},
\end{equation}
where $\mathbb{F}$ is the solution space, $f_{\mb{x}_k}$ is the predicted marginal p.d.f. of the $k$th generation, $g$ is the objective function to be maximized and $f_{w}(x|y)$ is the conditional p.d.f. decided by the mutation operator. Though the transition equation of marginal distributions loses information of interdependency between individuals, it has simpler form and can still provide a relatively complete description of the population. Moreover, as proved in\cite{Qi1}, it is accurate in the limiting case when the population size goes to infinity. Furthermore, in the second part of the research\cite{Qi2}, the authors analyzed the crossover operator and modified the transition equation to include all three operators in the simple EA. Overall, the studies of Qi et al. are inspiring, especially the idea of combining the modeling assumption that individuals are exchangeable with the mathematical analysis of point-wise convergence of p.d.f.s as the population size goes to infinity.

However, as will be shown in Section \ref{qiwrong}, the convergence proof for (\ref{eqn:intro-conve}) in\cite{Qi1} is \emph{problematic}. We provide a counterexample to show that in the authors' proof a key assertion about the law of large numbers (LLN) for exchangeable random vectors is generally not true. Therefore the whole proof is unsound. Furthermore, we show that the modeling assumption of exchangeability of individuals can not yield the transition equation in general. This means that under the authors' modeling assumption, the conclusion (\ref{eqn:intro-conve}) cannot be reached.

In addition to the aforementioned problems, we also show that the authors' proofs in both\cite{Qi1} and\cite{Qi2} are \emph{incomplete}. The authors did not address the convergence of the stacking of operators and of recursively iterating the algorithm. In essence, the authors only attempted to prove the convergence of the IPM for \emph{one} iteration step. Even if the proof for (\ref{eqn:intro-conve}) is correct, it only shows that the marginal p.d.f. of the $(k+1)$th population produced by the \emph{real} algorithm converges point-wisely to $f_{\mb{x}_{k+1}}(x)$ calculated by (\ref{eqn:intro-conve}), provided that the marginal p.d.f. of the $k$th generation is $f_{\mb{x}_{k}}(x)$ and assuming that the population size goes to infinity. However, this convergence does not automatically hold for all subsequent generations. In fact, it rarely holds because $f_{\mb{x}_{k+1}}(x)$ is only accurate in the limit. Compared with finite-sized populations produced by the real algorithm, it inevitably encompasses errors. As a result, (\ref{eqn:intro-conve}) cannot be iterated to make predictions for subsequent ($>k+1$) generations.

Besides \cite{Qi1,Qi2}, we found no other studies which attempted to prove the convergence of IPMs for EAs on \emph{continuous} space. Therefore, to fill the research gap, in Section \ref{framework} we propose a general analytical framework. The novelty of our framework is that from the very start of the analysis, we model generations of the population as random elements taking values in the metric space of infinite sequences, and we use convergence in distribution instead of point-wise convergence to define the convergence of IPMs.

To understand the issues and appreciate our framework, consider an EA operating in $\mathbb{R}^d$ on a fixed continuous optimization problem with different population sizes. When the population size is $n$, denote the algorithm by $\EA{n}$. The $k$th generation produced by $\EA{n}$ can be described by the joint distribution of $n$ random vectors of $\mathbb{R}^d$, with each random vector representing an individual. Denote the random element modeling the $k$th generation by $\Pkn=(\Xkin{k}{1}{n},\Xkin{k}{2}{n},\dots,\Xkin{k}{n}{n})$. Similarly, the same EA with population size $(n+1)$ is denoted by $\EA{n+1}$, and the $k$th generation it produces is modeled by $\PP{k}{n+1}=(\Xkin{k}{1}{n+1},\Xkin{k}{2}{n+1},\dots,\Xkin{k}{n+1}{n+1})$. Finally, denote the IPM for this EA by $\EA{\infty}$, and the generations it produces by $(\Pkinf)_{k=0,1,\dots}$. Notice that each $\Pkinf$ is a random sequence. Essentially, the convergence of IPMs requires that $\EA{\infty}$ predicts \emph{every} generation produced by $\EA{n}$ as $n\to \infty$. Mathematically, this corresponds to the requirement that for each generation $k$, the sequence $\PP{k}{1},\PP{k}{2},\dots$ converges to $\Pkinf$ in some sense as $n\to \infty$.

However, it is not obvious how one can rigorously define the convergence for the sequence $(\Pkn)_{n=1,2,\dots}$. This is because $\Pkn,n=1,2,\dots$ and the limit $\Pkinf$ are all random elements taking values in different metric spaces. The range of $\Pkn$ is the Cartesian product of $n$ copies of $\mathbb{R}^d$, whereas the range of $\Pkinf$ is the infinite product space $\mathbb{R}^d\times \mathbb{R}^d \times \dots$. To overcome this issue, Qi et al. essentially defined the convergence of IPMs as $\PP{k}{n} \tompw \PP{k}{\infty}$, where $\tompw$ stands for point-wise convergence of marginal p.d.f.s. However, as mentioned, we believe their proofs are problematic and incomplete.

In this research, we took a different approach. We extended $\Pkn$, unified the ranges of random elements in a common metric space and gave a mathematically rigorous definition of sequence convergence. We assume for each generation $k$, $\EA{n}$ first generates an intermediate infinite sequence of individuals $\Qkn=(\Xkin[y]{k}{1}{n}, \Xkin[y]{k}{2}{n},\dots)$ based on the previous generation $\PP{k-1}{n}$. Here $\Qkn$ is a random sequence whose elements are conditionally independent and identically distributed (c.i.i.d.) given $\PP{k-1}{n}$. Then, $\EA{n}$ preserves the first $n$ elements of $\Qkn$ to form the new generation $\Pkn$, i.e. $\Pkn = (\Xkin[y]{k}{1}{n},\Xkin[y]{k}{2}{n},\dots,\Xkin[y]{k}{n}{n})$. Basically the modified $\EA{n}$ progresses in the order of $\dots,\Qkn,\Pkn,\PP[Q]{k+1}{n},\PP{k+1}{n},\dots$. For $\EA{\infty}$, because $\Pkinf$ is already a random sequence, we just let $\Qkinf=\Pkinf$. Then, we define that $\EA{\infty}$ is convergent if and only if for every $k$, $\Qkn \tod \Qkinf$ as $n\to \infty$, where $\tod$ represents convergence in distribution, or equivalently weak convergence. Our design has several advantages. Firstly, for every population size $n$, the sequence $\Pkn,k=1,2,\dots$ coincides exactly with the population dynamics produced by $\EA{n}$ without the intermediate sequence $\Qkn,k=1,2,\dots$. In other words, our model is a faithful model and the intermediate step does not change the population dynamics. Secondly, the ranges of $\Qkn,n=1,2,\dots$ and $\Qkinf$ are unified in the same metric space. Therefore we can rigorously define the convergence of IPMs. Finally, in our proposed framework, the convergence of the stacking of operators and of iterating the algorithm can be proved. All these benefits come from the interplay between the finite-dimensional population dynamics $\Pkn$ and its infinite dimensional extensions $\Qkn$. The only modeling assumption in our framework is that new individuals are generated c.i.i.d. given the current generation. This is a reasonable assumption because exchangeability and c.i.i.d. are equivalent given the current population. We will present the framework and related topics in Section \ref{framework}.

To illustrate the effectiveness of our framework, we perform convergence analysis of IPM of the simple EA. As our analyses show, the modeling assumption of exchangeability cannot yield the transition equation. Therefore, to obtain meaningful results, we adopt a ``stronger'' modeling assumption that individuals of the same generation in the IPM are identically and independently distributed (i.i.d.). This assumption seems restricted at first sight, but it turns out to be a reasonable one. We analyze the mutation operator and the $k$-ary recombination operator. We show that these commonly used operators have the property of producing i.i.d. populations, in the sense that if the initial population is i.i.d., as the population size goes to infinity, in the limit all subsequent generations are also i.i.d.. This means that for these operators, the transition equation in the IPM can predict the real population dynamics as the population size goes to infinity. We also show that our results hold even if these operators are stacked together and iterated repeatedly by the algorithm. These results are presented in Section \ref{analysis}. Finally, in Section \ref{conclusion} we conclude the paper and propose future research.

To be complete, regarding \cite{Qi1,Qi2}, there is a comment from Yong\cite{Qi3} with reply. However, the comment was mainly about the latter part of \cite{Qi1}, where the authors analyzed the properties of EAs based on the IPM. It did not discuss the proof for the model itself. For IPMs of EAs on \emph{discrete} optimization problems, extensive research were done by Vose et al. in a series of studies\cite{vose1,vose2,vose3,vose4}. The problems under consideration were discrete optimization problems with \emph{finite} solution space. The staring point of the authors' analysis was to model each generation of the population as an ``incidence vector'', which describes for each point in the solution space the proportion of the population it occupies. Based on this representation the authors derived transition equations between incidence vectors of consecutive generations and analyzed their properties as the population size goes to infinity. However, for EAs on \emph{continuous} solution space, the analyses of Vose et al. are not immediately applicable. This is because for continuous optimization problems the solution space is not denumerable. Therefore, the population cannot be described by a finite-dimensional incidence vector.

\section{Discussion of the Works of Qi et al.}
\label{qiwrong}

\ifpdf
\graphicspath{{Chapter3/Figs/Raster/}{Chapter3/Figs/PDF/}{Chapter3/Figs/}}
\else
\graphicspath{{Chapter3/Figs/Vector/}{Chapter3/Figs/}}
\fi

In this section we analyze the results of Qi et al. in\cite{Qi1,Qi2}. We begin by introducing some preliminaries for the analysis. Then, in Section \ref{chapter:Qi-conve}, following the notations and derivations in the authors' papers, we provide a counterexample to show that the convergence proof for the transition equation in\cite{Qi1} is problematic. We further show that the modeling assumption of exchangeability cannot yield the transition equation in general. In Section \ref{chapter:Qi-stack}, we show that the analyses in both\cite{Qi1} and\cite{Qi2} are incomplete. The authors did not prove the convergence of IPMs in the cases where operators are stacked together and the algorithm is iterated for multiple generations.

\subsection{Preliminaries}
\label{chapter:Qi-preli}

In the authors' paper\cite{Qi1}, the problem to be optimized is
\begin{equation}
\label{eqn:Qi-obj}
\arg \max\limits_{x} g(x) \textnormal{ s.t. } x\in \mathbb{F} \subseteq \mathbb{R}^m,
\end{equation}
where $\mathbb{F}$ is the solution space and $g$ is some given objective function. The analysis intends to be general; therefore no explicit form of $g$ is assumed. The algorithm to be analyzed is the simple EA with proportionate selection and mutation. Let $\mb{X}_k=(\Xkn{k}{j})_{j=1}^N$ denote the $k$th generation produced by the EA, where $N$ is the population size. To generate the $(k+1)$th population, an intermediate population $\mb{X}_k^{\prime}=(\Xkn{k}{\prime j})_{j=1}^N$ is firstly generated based on $\mb{X}_k$ by the proportionate selection operator. The elements in $\mb{X}_k^{\prime}$ are c.i.i.d given $\mb{X}_k$. The distribution of $\mb{X}_k^{\prime}$ follows the conditional probability that
\begin{equation}
\label{eqn:Qi-sel}
\textnormal{P}(\Xkn{k}{\prime i}=\Xkn{k}{j}|\mb{X}_k)=\frac{g(\Xkn{k}{j})}{\sum_{l=1}^{N}g(\Xkn{k}{l})}, \textnormal{ for all } i,j=1,2,\dots,N.
\end{equation}
After selection, each individual in $\mb{X}_k^{\prime}$ is mutated to generate individuals in $\mb{X}_{k+1}$. The mutation is conducted following the conditional p.d.f.
\begin{equation}
\label{eqn:Qi-mutat}
f(\Xkn{k+1}{i}=x|\Xkn{k}{\prime i}=y)=f_w(x|y).
\end{equation}
Overall the algorithm is illustrated in Fig. \ref{fig:Qi-simpleea}.

\begin{figure}[ht]
	\centering
	\small
	\begin{framed}
		\begin{algorithmic}[1]
			\Require{population size $N$; p.d.f. of the initial population $f_{\mb{x}_0}$}
			\State $k\gets 0$
			\State{sample the $N$ i.i.d. individuals $\Xkn{k}{1},\Xkn{k}{2},\ldots,\Xkn{k}{N}$ according to $f_{\mb{x}_0}$}
			\While{stopping criteria is not satisfied}
			\State \begin{varwidth}[t]{\linewidth-\algorithmicindent}
				select $\Xkn{k}{\prime 1},\Xkn{k}{\prime 2},\dots,\Xkn{k}{\prime N}$ from $\Xkn{k}{1},\Xkn{k}{2},\ldots,\Xkn{k}{N}$ identically and independently according to the probability that
				\[\textnormal{P}(\Xkn{k}{\prime i}=\Xkn{k}{j}|\mb{X}_k)=\frac{g(\Xkn{k}{j})}{\sum_{l=1}^{N}g(\Xkn{k}{l})}, \forall i,j=1,2,\dots,N\]\par
				\Comment Selection
			\end{varwidth}
			\State
			\begin{varwidth}[t]{\linewidth-\algorithmicindent}
				perturb $\Xkn{k}{\prime 1},\Xkn{k}{\prime 2},\dots,\Xkn{k}{\prime N}$ to form the new generation $\Xkn{k+1}{1},\Xkn{k+1}{2},\ldots,\Xkn{k+1}{N}$ according to the common conditional p.d.f.
				\[f(\Xkn{k+1}{i}=x|\Xkn{k+1}{\prime i}=y)=f_w(x|y), \forall i=1,2,\dots,N\]\par
				\Comment Mutation
			\end{varwidth}
			\State $k\gets k+1$
			\EndWhile
		\end{algorithmic}
	\end{framed}
	\caption{The pseudocode of the simple EA}
	\label{fig:Qi-simpleea}
\end{figure}

After presenting the optimization problem and the algorithm, the authors proved the convergence of the IPM. It is the main result in\cite{Qi1}. It can be reiterated as follows.

\begin{theorem}[Theorem 1 in Qi et al.\cite{Qi1}]
	\label{thm:Qi-conv}
	Assume that the fitness function $g(x)$ in (\ref{eqn:Qi-obj}) and the mutation operator of simple EA described by (\ref{eqn:Qi-mutat}) satisfy the following conditions:
	\begin{enumerate}
		\item $0<g_{\min}\leq g(x) \leq g_{\max}< \infty,\forall x \in \mathbb{F}$. \label{eqn:Qi-condi1}
		\item $\sup\limits_{x,y\in \mathbb{R}^d} f_w(x|y)\leq M<\infty$. \label{eqn:Qi-condi2}
	\end{enumerate}
	Then as $n\to \infty$, the time history of the simple EA can be described by a sequence of random vectors $(\mb{x}_k)_{k=0}^{\infty}$ with densities
	\begin{equation}
	\label{eqn:Qi-trans}
	f_{\mb{x}_{k+1}}(x) = \frac{\int_{\mathbb{F}} \! f_{\mb{x}_{k}}(y) g(y) f_{w}(x|y) \, \mathrm{d}y}{\int_{\mathbb{F}} \! f_{\mb{x}_{k}}(y)g(y) \, \mathrm{d}y}.
	\end{equation}
\end{theorem}

In Theorem \ref{thm:Qi-conv}, $f_{\mb{x}_k}$ is the marginal p.d.f. of the $k$th generation predicted by the IPM.

As the proof for Theorem \ref{thm:Qi-conv} in\cite{Qi1} and the analyses in this paper use the concept of exchangeability in probability theory, we list its definition and some basic facts.

\begin{definition}[Exchangeable random variables, Definition 1.1.1 in\cite{exch1}]
	\label{def:Qi-exchangeable}
	A finite set of random variables $\{\mb{x}_i\}_{i=1}^n$ is said to be exchangeable if the joint distribution of $(\mb{x}_i)_{i=1}^n$ is invariant with respect to permutations of the indices $1,2,\dots,n$. A collection of random variables $\{\mb{x}_\alpha:\alpha \in \operatorname{\Gamma}\}$ is said to be exchangeable if every finite subset of $\{\mb{x}_\alpha:\alpha \in \operatorname{\Gamma}\}$ is exchangeable.
\end{definition}

Definition \ref{def:Qi-exchangeable} can also be extended to cover exchangeable random vectors or exchangeable random elements by replacing the term ``random variables'' in the definition with the respective term. One property of exchangeability is that if $\{\mb{x}_i\}_{i=1}^n$ are $n$ exchangeable random elements, then the joint distributions of any $1\leq k\leq n$ distinct ones of them are always the same (Proposition 1.1.1 in\cite{exch1}). When $k=1$ this property indicates that $\{\mb{x}_i\}_{i=1}^n$ have the same marginal distribution. Another property is that a collection of random elements are exchangeable if and only if they are c.i.i.d. given some \sigmafield/ $\mathcal{G}$ (Theorem 1.2.2 in\cite{exch1}). Conversely, a collection of c.i.i.d. random elements are always exchangeable. Finally, an obvious fact is that i.i.d. random elements are exchangeable, but the converse is not necessarily true.

It can be seen that the simple EA generates c.i.i.d. individuals given the current population. Therefore, the individuals within the same generation are exchangeable, and they have the same marginal distribution. This leads to the transition equation of marginal p.d.f.s in Theorem \ref{thm:Qi-conv}. To analyze its proof and construct our framework, we will also use the definition and properties of exchangeability.

\subsection{Convergence Proof of the Transition Equation}
\label{chapter:Qi-conve}

In this section we analyze the proof of Theorem \ref{thm:Qi-conv} and show that it is incorrect. The proof by Qi et al. is in Appendix A of\cite{Qi1}. In the proof the authors assumed that individuals in the same generation are exchangeable, therefore they have the same marginal distribution. After a series of derivation steps, the authors managed to obtain a transition equation between the density functions of $\Xkn{k+1}{i}$ and $\mb{X}_k$:
\begin{align}
f_{\Xkn{k+1}{i}}(x)
=&\iint_{\mathbb{F}^N} \! \frac{{g({y_j}){f_w}(x|{y_j})}}{{\frac{1}{N}\sum\limits_{l = 1}^N {g({y_l})} }}{f_{\mb{X}_k}}({y_1},{y_2}, \ldots ,{y_n})\,\nonumber\\
&\mathrm{d}{y_1} \mathrm{d}{y_2} \ldots \mathrm{d}{y_n} \textnormal{ for any fixed } i,j \label{eqn:Qi-tr1}\\
=&\EXP\left[ {\frac{{{\mb{\xi} _k}(x)}}{{\mb{\eta} _k^N}}} \right], \label{eqn:Qi-tr2}
\end{align}
where in (\ref{eqn:Qi-tr2}),
\begin{align}
\mb{\xi}_k(x) &\triangleq g(\Xkn{k}{j})f_w(x|\Xkn{k}{j}) \textnormal{ for any fixed }j, \label{eqn:Qi-xik}\\
\mb{\eta}_k^N &\triangleq \frac{1}{N} \sum\limits_{l=1}^{N}g(\Xkn{k}{l}). \label{eqn:Qi-etakn}
\end{align}

(\ref{eqn:Qi-tr1}) and (\ref{eqn:Qi-tr2}) are exact. They accurately describe how the marginal p.d.f. for any individual in the next generation can be calculated from the joint p.d.f. of individuals in the current generation. Noticing that $\mb{\eta}_k^N$ is the average of the exchangeable random variables $\{g(\Xkn{k}{j})\}_{j=1}^{N}$, by the LLN for exchangeable random variables, the authors asserted that
\begin{equation}
\label{eqn:Qi-etakn=etak}
\lim\limits_{N\to \infty}\mb{\eta}_k^N=\mb{\eta}_k \textnormal{ almost surely (a.s.)}.
\end{equation}
The authors further asserted that $\mb{\eta}_k$ is itself a random variable, satisfying
\begin{equation}
\label{eqn:Qi-Eetak}
\EXP [ \mb{\eta}_k ]=\EXP [ g(\Xkn{k}{j}) ] \textnormal{ for any }j.
\end{equation}

(\ref{eqn:Qi-etakn=etak}) and (\ref{eqn:Qi-Eetak}) correspond to (A13) and (A14) in Appendix A of\cite{Qi1}, respectively. The authors' proof is  \emph{correct} until this step. However, the authors then asserted that
\begin{quotation}
	\noindent $\mb{\eta}_k$ is independent of $\mb{\eta}_k^N$ for any finite N. In particular, $\mb{\eta}_k$ is independent of $\mb{\eta}_k^1=g(\mb{x}_k^j)$ for all $j=1,2,\dots,N$.
\end{quotation}\eqnum\label{eqn:Qi-assertX}

Based on this assertion the authors then proved that for all $k$ and $x$,
\begin{equation}
\label{eqn:Qi-qilimit}
\lim_{N\to 0}
\left|
\EXP \left[ \frac{\mb{\xi}_k (x)}{\mb{\eta}_k^N} \right] -
\frac{ \EXP \left[ \mb{\xi}_k (x) \right] }{\EXP \left[ \mb{\eta}_k \right]}
\right| = 0.
\end{equation}
Therefore, the p.d.f. in (\ref{eqn:Qi-tr2}) converges point-wisely to $\frac{ \EXP \left[ \mb{\xi}_k (x) \right] }{\EXP \left[ \mb{\eta}_k \right]}$. Noticing that the expression of $\frac{ \EXP \left[ \mb{\xi}_k (x) \right] }{\EXP \left[ \mb{\eta}_k \right]}$ is equal to the right hand side of (\ref{eqn:Qi-trans}), the authors claimed that Theorem \ref{thm:Qi-conv} is proved.  

In the following, we provide a counterexample to show that assertion (\ref{eqn:Qi-assertX}) is not true. Then, we carry out further analysis to show that under the modeling assumption of exchangeability, the conclusion in (\ref{eqn:Qi-qilimit}), or equivalently Theorem \ref{thm:Qi-conv}, cannot be true in general.

\subsubsection{On Assertion (\ref{eqn:Qi-assertX})}\label{chapter:Qi-onassert}
We first reformulate the assertion. Since $\{\Xkn{k}{l}\}_{l=1}^N$ are exchangeable, $\{g(\Xkn{k}{l})\}_{l=1}^N$ are exchangeable (Property 1.1.2 in\cite{exch1}). 
Let $\mb{y}_l = g(\mb{x}_k^l),l=1,\dots,N$. Then the premises of Theorem \ref{thm:Qi-conv} are equivalent to
\begin{equation}
\label{eqn:Qi-ylcondi1}
\{\mb{y}_l\}_{l=1}^N \textnormal{ are exchangeable and }g_{\min} \leq \mb{y}_l \leq g_{\max}.
\end{equation}
Let $\mb{y}=\mb{\eta}_k$. According to (\ref{eqn:Qi-etakn}), (\ref{eqn:Qi-etakn=etak}) and (\ref{eqn:Qi-Eetak}), $\mb{y}$ has the properties that
\begin{empheq}[left=\empheqlbrace]{align}
\lim\limits_{N\to \infty} \frac{1}{N} \sum_{l=1}^{N} \mb{y}_l = \mb{y}, \textnormal{ a.s.}, \label{eqn:Qi-ylcondi2}
\\
\EXP(\mb{y})=\EXP(\mb{y}_l) \textnormal{ for any }l. \label{eqn:Qi-ylcondi3}
\end{empheq}
Since $g$ is a general function, there is no other restrictions for $\{\mb{y}_l\}_{l=1}^N$ and $\mb{y}$. Therefore, (\ref{eqn:Qi-assertX}) is equivalent to the following assertion:
\begin{quote}
	For any $\{\mb{y}_l\}_{l=1}^N$ and $\mb{y}$ satisfying (\ref{eqn:Qi-ylcondi1}), (\ref{eqn:Qi-ylcondi2}) and (\ref{eqn:Qi-ylcondi3}), $\mb{y}$ and $\frac{1}{N} \sum_{l=1}^{N} \mb{y}_l$ are independent for any finite $N$. In particular, $\mb{y}$ is independent of $\mb{y}_l$ for any $l$.
\end{quote}\eqnum\label{eqn:Qi-assertY}

However, we use the following counterexample (modified from Example 1.1.1 and related discussions on pages 11-12 in\cite{exch1}) to show that assertion (\ref{eqn:Qi-assertY}) is not true. Therefore (\ref{eqn:Qi-assertX}) is not true.

\subsubsection{Counterexample} Let $\{\mb{z}_l\}_{l=1}^{\infty}$ be a sequence of i.i.d. random variables satisfying
\[-\frac{g_{\max}-g_{\min}}{4} \leq \mb{z}_l \leq \frac{g_{\max}-g_{\min}}{4} \textnormal{ and } \EXP(\mb{z}_l)=0\]
for all $l$. Let $\mb{y}$ be a random variable \emph{independent} of $\{\mb{z}_l\}_{l=1}^{\infty}$ satisfying
\[\frac{g_{\max}+3 g_{\min}}{4} \leq \mb{y} \leq \frac{3 g_{\max}+g_{\min}}{4}\] and \[\EXP(\mb{y})=\frac{g_{\max}+g_{\min}}{2}.\]
Finally, let $\mb{y}_l=\mb{z}_l+\mb{y}$ for all $l$.

It can easily be verified that $\{\mb{y}_l\}_{l=1}^{\infty}$ and $\mb{y}$ satisfy (\ref{eqn:Qi-ylcondi1}) and (\ref{eqn:Qi-ylcondi3}). Since $\mb{z}_l$ is bounded, $\EXP(|\mb{z}_l |)<\infty$ for any $l$. By the strong law of large numbers (SLLN) for i.i.d. random variables,
\[\frac{1}{N} \sum_{l=1}^{N} \mb{z}_l\to 0 \textnormal{ a.s. as } N \to \infty.\]
Therefore (\ref{eqn:Qi-ylcondi2}) is also satisfied, i.e. $\mb{y}$ is the limit of $\frac{1}{N} \sum_{l=1}^{N} \mb{y}_l$ as $N\to \infty$. However, because $\frac{1}{N} \sum_{l=1}^{N} \mb{y}_l=\mb{y}+\frac{1}{N} \sum_{l=1}^{N} \mb{z}_l$ and $\mb{y}$ is independent of $\{\mb{z}_l\}_{l=1}^{\infty}$, it can be seen that $\frac{1}{N} \sum_{l=1}^{N} \mb{y}_l$ is \emph{not} independent of $\mb{y}$ except for some degenerate cases (for example when $\mb{y}$ equals to a constant). In particular, in general $\mb{y}_l=\mb{y}+\mb{z}_l$ is \emph{not} independent of $\mb{y}$ for any $l$. Therefore, assertion (\ref{eqn:Qi-assertY}) is not true. Equivalently, assertion (\ref{eqn:Qi-assertX}) is not true.

\subsubsection{Further Analysis}
\label{chapter:Qi-futher}

In\cite{Qi1} the authors intended to prove Theorem \ref{thm:Qi-conv}, or equivalently (\ref{eqn:Qi-qilimit}). As shown by the counterexample, assertion (\ref{eqn:Qi-assertX}) is not true. This renders the authors' proof for (\ref{eqn:Qi-qilimit}) invalid.

In the following, we carry out further analysis to show that (\ref{eqn:Qi-qilimit}) cannot be true even considering other methods of proof and adding new sufficient conditions. Therefore, in general, Theorem \ref{thm:Qi-conv} cannot be true.

To begin with, consider the random variable ${\frac{{{\mb{\xi} _k}(x)}}{{\mb{\eta} _k^N}}}$. We prove the following lemma.

\begin{lemma}
	\label{lemma:Qi-convExietak}
	$\EXP\left( {\frac{{{\mb{\xi} _k}(x)}}{{\mb{\eta} _k^N}}} \right) \to \EXP\left( {\frac{{{\mb{\xi} _k}(x)}}{{\mb{\eta} _k}}} \right)$ as $N \to \infty$.
\end{lemma}

\begin{proof}
	According to (\ref{eqn:Qi-etakn=etak}), $\mb{\eta}_k^N \toas \mb{\eta}_k$. Since $g_{\min} \leq \mb{\eta}_k^N \leq g_{\max}$, $0<g_{\min} \leq \mb{\eta}_k \leq g_{\max}$ almost surely.
	
	Since $h(x)=\frac{1}{x}$ is continuous on $(0,\infty)$, we have
	\[		
	h(\mb{\eta}_k^N)={\frac{1}{{\mb{\eta} _k^N}}} \toas 
	h(\mb{\eta}_k)={\frac{1}{{\mb{\eta} _k}}} \textnormal{ (Proposition 47.2 in\cite{prob})}.
	\]
	
	Then we have
	\[
	{\frac{\mb{\xi}_k (x)}{{\mb{\eta} _k^N}}} \toas {\frac{\mb{\xi}_k (x)}{{\mb{\eta} _k}}}
	\textnormal{ (Proposition 47.4 (ii) in\cite{prob})}.
	\]
	
	Finally, by the conditions in Theorem \ref{thm:Qi-conv}, $0 \leq \frac{\mb{\xi}_k (x)}{{\mb{\eta} _k^N}} \leq \frac{M g_{\max}}{g_{\min}}$. By the Lebesgue's Dominated Convergence Theorem (Proposition 11.30 in\cite{prob}), we have $\EXP\left( {\frac{\mb{\xi}_k (x)}{{\mb{\eta} _k^N}}} \right) \to
	\EXP\left( {\frac{\mb{\xi}_k (x)}{{\mb{\eta} _k}}} \right)$ as $N\to \infty$.	
\end{proof}

Now by Lemma \ref{lemma:Qi-convExietak}, (\ref{eqn:Qi-qilimit}) is equivalent to
\[
\label{eqn:Qi-ExikDivEetak}
\EXP\left( {\frac{\mb{\xi}_k (x)}{{\mb{\eta} _k}}} \right) = \frac{ \EXP \left[ \mb{\xi}_k (x) \right] }{\EXP \left[ \mb{\eta}_k \right]}. \tag{$\triangle$}
\]

Now it is clear that if the only assumption is exchangeability, ($\triangle$) is \emph{not} true even considering other methods of proof. Of course, if (\ref{eqn:Qi-assertX}) is true, $\mb{\xi}_k(x)$ and $\mb{\eta}_k$ are independent, then (\ref{eqn:Qi-ExikDivEetak}) is true. However, as already shown by the counterexample, (\ref{eqn:Qi-assertX}) is not true in general. Therefore, (\ref{eqn:Qi-ExikDivEetak}), and equivalently Theorem \ref{thm:Qi-conv}, are in general not true.

A natural question then arises: is it possible to introduce some reasonable sufficient conditions such that ($\triangle$) can be proved? One of such conditions frequently used is that $\mb{\eta}_k=\EXP[ g(\Xkn{k}{j})]$, i.e. $\mb{\eta}_k^N$ converges to its expectation, a constant which equals $\EXP[ g(\Xkn{k}{j})]$ for any $j$. However, the following analysis shows that given the modeling assumption of exchangeability, this condition is not true in general. Therefore it cannot be introduced.

For exchangeable random variables $\{g(\Xkn{k}{l})\}_{l=1}^N$, we have
\begin{align}
\VAR(\mb{\eta}_k) =& \lim\limits_{N\to \infty} \VAR(\mb{\eta}_k^N) \label{eqn:Qi-further-step1}\\
=&\lim\limits_{N\to \infty} \left\{
\EXP{\left[ \frac{\sum\limits_{l=1}^{N} g(\Xkn{k}{l})}{N} \right]}^2-
{\left[ \EXP\frac{\sum\limits_{l=1}^{N} g(\Xkn{k}{l})}{N} \right]}^2
\right\} \nonumber\\
=&\lim\limits_{N\to \infty} \frac{1}{N^2}\left\{
\sum\limits_{l=1}^{N} \VAR \left[ g(\Xkn{k}{l}) \right] \right. \nonumber +\\
&\left. \sum\limits_{i = 1 }^{N} \sum\limits_{j = 1,j\neq i }^{N} \COV \left[ g(\Xkn{k}{i}), g(\Xkn{k}{j}) \right]
\right\} \nonumber\\
=&\lim\limits_{N \to \infty} \left\{
\frac{\VAR \left[ g(\Xkn{k}{1}) \right]}{N} +
\frac{N-1}{N} \COV \left[ g(\Xkn{k}{1}),g(\Xkn{k}{2}) \right]
\right\} \label{eqn:Qi-further-step2}\\
=&\COV \left[ g(\Xkn{k}{1}),g(\Xkn{k}{2}) \right], \label{eqn:Qi-further-step3}
\end{align}
where $\VAR(\mb{x})$ is the variance of $\mb{x}$ and $\COV(\mb{x},\mb{y})$ is the covariance of $\mb{x}$ and $\mb{y}$. (\ref{eqn:Qi-further-step1}) is by the boundedness of $\mb{\eta}_k^N$ and the Lebesgue's Dominated Convergence Theorem, (\ref{eqn:Qi-further-step2}) is by the exchangeability of $\{ \Xkn{k}{j} \}_{j=1}^N$, and (\ref{eqn:Qi-further-step3}) is by the boundedness of $g$ and pushing $N$ to infinity. Now it is clear that if the only modeling assumption is exchangeability, there is no guarantee that $\COV \left[ g(\Xkn{k}{1}),g(\Xkn{k}{2}) \right] = 0$. Therefore, in general $\mb{\eta}_k^N$ does not converge to a constant. Thus this condition cannot be introduced as a sufficient condition in order to prove (\ref{eqn:Qi-ExikDivEetak}).

\subsubsection{Summary}
\label{chapter:Qi-conv-summary}

As the analyses in this section show, the transition equation (\ref{eqn:Qi-trans}) does not hold under the modeling assumption of exchangeability. However, it does not preclude the possibility of enhancing the modeling assumption so that it can yield analytical results similar to the transition equation (\ref{eqn:Qi-trans}). We deal with this issue by adopting the ``stronger'' i.i.d. assumption when building IPMs. However, before presenting our framework and analyses, we show why the proofs in both\cite{Qi1} and\cite{Qi2} are \emph{incomplete}.

\subsection{The Issue of the Stacking of Operators and Iterating the Algorithm}
\label{chapter:Qi-stack}

In the following, we discuss IPMs from another perspective. Consider an EA with only one operator. Let the operator be denoted by $\operatorname{H}$. When the population size is $n$, denote this EA by $\EA{n}$ and the operator it \emph{actually} uses by $\operatorname{H}_n$. Let $\Pkn=(\Xkin{k}{i}{n})_{i=1}^n$ denote the $k$th generation produced by $\EA{n}$. Then the transition rules between consecutive generations produced by $\EA{n}$ can be described by $\PP{k+1}{n}=\operatorname{H}_n (\Pkn)$. In Table \ref{tab:Qi-iterating}, we write down the population dynamics of $\EA{n}$. Each row in Table \ref{tab:Qi-iterating} shows the population dynamics produced by $\EA{n}$. In the table $\PP{k}{n}$ is expanded as $[\operatorname{H}_n]^k(\PP{0}{n})$. Let $\EA{\infty}$ denote the IPM, and $\Pkinf=[\operatorname{H}_{\infty}]^k(\PP{0}{\infty})$ denote the populations predicted by $\EA{\infty}$. Then we can summarize the results in\cite{Qi1} in the following way.
\begin{table}
	\renewcommand{\arraystretch}{1.5}
	\caption{Population dynamics of $\EA{n}$ under operator $\operatorname{H}$}
	\label{tab:Qi-iterating}
	\centering
	\begin{tabular}{l | llll}
		\backslashbox{$\EA{n}$}{k} & $0$              & $1$                                         & $2$                                                                    & $\dots$ \\ \hline
		$\EA{1}$                   & $\PP{0}{1}$      & $\operatorname{H}_1(\PP{0}{1})$             & $\operatorname{H}_1(\operatorname{H}_1(\PP{0}{1}))$                    & \dots   \\
		\vdots                     & \vdots           & \vdots                                      & \vdots                                                                 &  \\
		$\EA{n}$                   & $\PP{0}{n}$      & $\operatorname{H}_n(\PP{0}{n})$             & $\operatorname{H}_n(\operatorname{H}_n(\PP{0}{n}))$                    & \dots   \\
		\vdots                     & \vdots           & \vdots                                      & \vdots                                                                 &  \\
		$\EA{\infty}$              & $\PP{0}{\infty}$ & $\operatorname{H}_{\infty}(\PP{0}{\infty})$ & $\operatorname{H}_{\infty}(\operatorname{H}_{\infty}(\PP{0}{\infty}))$ & \dots
	\end{tabular}
\end{table}

Let $\operatorname{H}$ represent the \emph{combined} operator of proportionate selection and mutation. Though the authors originally developed the transition equation from the $k$th to the $(k+1)$th generation, without loss of generality we can consider only the populations from the initial generation to the onward ones. Assume that the initial population comes from a \emph{known} sequence of individuals, represented by $\mb{P}_0={(\mb{x}_i)}_{i=1}^{\infty}$. For $\EA{n}$, its initial population $\PP{0}{n}$ consists of the first $n$ elements of $\mb{P}_0$, i.e. $\PP{0}{n}=(\mb{x}_i)_{i=1}^n$. Let $\PP{0}{\infty}=\mb{P}_0$. This setting represents the fact that $\EA{n}$ uses the same initial population, and $\EA{\infty}$ knows this exact initial population. The aim of $\EA{\infty}$ is to predict the subsequent populations. Considering that $\PP{0}{n}$ and $\PP{0}{\infty}$ are all from $\mb{P}_0$, if we redefine $\operatorname{H}_n$ to be operators on $\mb{P}_0$ which only takes the first $n$ elements to produce the next generation, then the authors essentially proved that
\begin{equation}
\label{eqn:Qi-onestep}
\operatorname{H}_{n}(\mb{P}_0) \tompw \operatorname{H}_{\infty}(\mb{P}_0) \textnormal{ as } n \to \infty,
\end{equation}
where m.p.w. stands for point-wise convergence of marginal p.d.f.s.

However, apart from the fact that this proof is problematic, the authors' proof covers only \emph{one} iteration step, corresponding to the column-wise convergence of the $k=1$ column in Table \ref{tab:Qi-iterating}. The problem is that even if (\ref{eqn:Qi-onestep}) is true, it does not automatically lead to the conclusion that for the arbitrary $k$th step, $[\operatorname{H}_{n}]^k(\mb{P}_0)\tompw [\operatorname{H}_{\infty}]^k(\mb{P}_0)$ as $n\to \infty$. In other words, one has to study whether the transition equation for one step can be iterated recursively to predict the populations after multiple steps. In Table \ref{tab:Qi-iterating}, this problem corresponds to whether other columns have similar column-wise convergence property when the convergence of the $k=1$ column is proved.

To give an example, consider the column of $k=2$ in Table \ref{tab:Qi-iterating}. To prove column-wise convergence, the authors need to prove that given (\ref{eqn:Qi-onestep}),
\begin{align}
&\operatorname{H}_{n}(\PP{1}{n})\tompw \operatorname{H}_{\infty}(\PP{1}{\infty}) \textnormal{, or equivalently} \label{eqn:Qi-diffinput}\\
&[\operatorname{H}_{n}]^2(\mb{P}_0)\tompw [\operatorname{H}_{\infty}]^2(\mb{P}_0) \label{eqn:Qi-diffop}
\end{align}
as $n \to \infty$. Comparing (\ref{eqn:Qi-onestep}) with (\ref{eqn:Qi-diffinput}) and (\ref{eqn:Qi-diffop}), (\ref{eqn:Qi-diffinput}) has the same sequence of operators but with a sequence of converging inputs, (\ref{eqn:Qi-diffop}) has the same input but with a sequence of different operators. Therefore, they are not necessarily true even if (\ref{eqn:Qi-onestep}) is proved. In fact, different techniques may have to be adopted to prove (\ref{eqn:Qi-onestep}) and (\ref{eqn:Qi-diffinput}), or equivalently (\ref{eqn:Qi-onestep}) and (\ref{eqn:Qi-diffop}). Similar problem exists when considering the arbitrary $k$th generation. We call this problem the issue of iterating the algorithm. As Qi et al. in both\cite{Qi1,Qi2} ignored this issue, we believe their proofs are \emph{incomplete}.

The issue of the stacking of operators is similar. Given some operator $\operatorname{H}$ satisfying (\ref{eqn:Qi-onestep}) and some operator $\operatorname{G}$ satisfying
\[\operatorname{G}_n(\mb{P}_0) \tompw \operatorname{G}_{\infty}(\mb{P}_0)\]
as $n \to \infty$, it is \emph{not} necessarily true that
\[\operatorname{H}_n(\operatorname{G}_n(\mb{P}_0)) \tompw \operatorname{H}_{\infty}(\operatorname{G}_{\infty}(\mb{P}_0))\]
as $n\to \infty$. However, the authors in\cite{Qi2} totally ignored this issue and combined the transition equations for selection, mutation and crossover together (in Section III of\cite{Qi2}) without any justification.

In addition, there are several statements in the authors' proofs in\cite{Qi2} that are questionable. First, in the first paragraph of Appendix A (the proof for Theorem 1 in that paper), the authors considered a pair of parents $\mb{x}_k$ and $\mb{x}_k^\prime$ for the uniform crossover operator. $\mb{x}_k$ and $\mb{x}_k^\prime$ are ``drawn from the population independently with the same density of $f_{\mb{x}_k} \equiv f_{\mb{x}_k^\prime}$''. Then, the authors claimed that ``the joint density of $\mb{x}_k$ and $\mb{x}_k^\prime$ is therefore $f_{\mb{x}_k} \cdot f_{\mb{x}_k^\prime}$''. This is simply not true. Two individuals drawn independently from the same population are \emph{conditionally} independent, they are not necessarily independent, unless the modeling assumption is that all individuals in the same population are independent. In fact, without the i.i.d. assumption, it is very likely that individuals in the same population are dependent. Therefore, the joint density function of $\mb{x}_k$ and $\mb{x}_k^\prime$ is not necessarily $f_{\mb{x}_k} \cdot f_{\mb{x}_k^\prime}$, and the authors' proof for Theorem 1 in\cite{Qi2} is dubious at best. On the other hand, even if the authors' modeling assumption is independence of individuals for the uniform crossover operator, this assumption is incompatible with the modeling assumption of exchangeability in\cite{Qi1} for the operators of selection and mutation. Therefore, combining the transition equations for all these three operators is problematic, because the assumption of independence cannot hold beyond one iteration step.

Another issue in\cite{Qi2} is that the uniform crossover operator produces two \emph{dependent} offspring at the same time. As a result, after uniform crossover, the intermediate population is not even exchangeable because it has pair-wise dependency between individuals. Then the same problem arises, that is the transition equation for the uniform crossover operator cannot be combined with the transition equations for selection and mutation. This is because the uniform crossover operator produces intermediate populations without exchangeability, but this property is required for modeling selection and mutation. Besides, the transition equation for the uniform crossover operator cannot be iterated beyond one step. This is because regardless of independence or exchangeability as its modeling assumption, this assumption will surely be corrupted beyond one iteration step.

In summary, several issues arise from previous studies on IPMs for EAs on continuous optimization problems. Therefore, new frameworks and proof methods are needed for analyzing the convergence of IPMs and studying the issue of the stacking of operators and iterating the algorithm.

\section{Proposed Framework}
\label{framework}

In this section, we present our proposed analytical framework. In constructing the framework we strive to achieve the following three goals.
\begin{enumerate}
	\item The framework should be general enough to cover real world operators and to characterize the evolutionary process of real EA.
	\item The framework should be able to define the convergence of IPMs and serve as justifications of using them. The definition should match people's intuitions and at the same time be mathematically rigorous.
	\item The framework should provide an infrastructure to study the issue of the stacking of operators and iterating the algorithm.
\end{enumerate}

The contents of this section roughly reflects the pursuit of the first two goals. The third goal is reflected in the analyses of the simple EA in Section \ref{analysis}. More specifically, in Section \ref{chapter:framework-preli}, we introduce notations and preliminaries for the remainder of this paper. In Section \ref{chapter:framework-model}, we present our framework. In the framework, each generation is modeled by a random sequence. This approach unifies the spaces of random elements modeling populations of different sizes. In Section \ref{chapter:framework-conv}, we define the convergence of the IPM as convergence in distribution on the space of random sequences. We summarize and discuss our framework in Section \ref{chapter:framework-summary}.

To appreciate the significance of our framework, it is worth reviewing the methodology in \cite{Qi1,Qi2} studying the convergence of IPMs. Implicitly, the authors in \cite{Qi1,Qi2} used point-wise convergence of marginal p.d.f. as the criteria of defining the convergence of IPMs. Apart from the proofs being problematic and incomplete, this definition does not consider the \emph{joint} distribution of individuals of the population. Thus, it loses information and cannot characterize the dynamics of the whole population. Besides, point-wise convergence of p.d.f.s depends on the existence and the explicit forms of the p.d.f.s. This fact limits the generality of the methodology. In addition, compared with convergence in distribution used in this paper, the criteria of point-wise convergence is unnecessarily strict. In essence, the core of the criteria should characterize the similarity between distributions of random elements. In this regard, convergence in distribution matches the intuition and suffices for the task. A stronger criteria, such as point-wise convergence, will inevitably increase the difficulty in analysis. Finally, in this paper we separate the framework (the definition of the convergence of IPMs) from the analyses of operators. The organization is logical and general. 

\subsection{Notations and Preliminaries}
\label{chapter:framework-preli}

In the remainder of this paper we focus on the unconstrained continuous optimization problem
\begin{equation}
\label{eqn:framework-obj}
\arg \max\limits_{x} g(x) \textnormal{ s.t. } x\in \mathbb{R}^d,
\end{equation}
where $g$ is some given objective function. Our framework is general enough such that it does not require other conditions on the objective function $g$. However, to prove the convergence of IPMs for mutation and recombination, conditions such as those in Theorem \ref{thm:Qi-conv} are sometimes needed. We will introduce them when they are required.

From now on we use $\mathbb{N}$ to denote the set of nonnegative integers and $\mathbb{N}_+$ the set of positive integers. For any two real numbers $a$ and $b$, let $a \wedge b$ be the smaller one of them and $a \vee b$ be the larger one of them. Let $\mb{x},\mb{y}$ be random elements of some measurable space $(\Omega,\mathcal{F})$. We use $\Law(\mb{x})$ to represent the law of $\mb{x}$. If $\mb{x}$ and $\mb{y}$ follow the same law, i.e. $\Prob(\mb{x}\in A)=\Prob(\mb{y}\in A)$ for every $A\in \mathcal{F}$, we write $\Law(\mb{x})=\Law(\mb{y})$. Note that $\Law(\mb{x})=\Law(\mb{y})$ and $\mb{x}=\mb{y}$ have different meanings. In particular, $\mb{x}=\mb{y}$ indicates dependency between $\mb{x}$ and $\mb{y}$.

We use the notation $(x_i)_{i=m}^n$ to represent the array $(x_m,x_{m+1},\dots,x_n)$. When $n=\infty$, $(x_i)_{i=m}^\infty$ represents the infinite sequence $(x_m,x_{m+1},\dots)$. We use $\{x_i\}_{i=m}^n$ and $\{x_i\}_{i=m}^\infty$ to represent the collections $\{x_m,x_{m+1},\dots,x_n\}$ and $\{x_i|i=m,m+1,\dots\}$, respectively. When the range is clear, we use $(x_i)_i$ and $\{x_i\}_i$ or $(x_i)$ and $\{x_i\}$ for short.

Let $\mathbb{S}$ denote the solution space $\mathbb{R}^d$. This simplifies our notation system when we discuss the spaces $\mathbb{S}^n$ and $\mathbb{S}^\infty$. In the following, we define metrics and \sigmafield/s on $\mathbb{S}$, $\mathbb{S}^n$ and $\mathbb{S}^\infty$ and state properties of the corresponding measurable spaces.

$\mathbb{S}$ is equipped with the ordinary metric $\rho(x,y)=[\sum_{i=1}^{d}(x_i-y_i)^2]^{\frac{1}{2}}$. Let $\mathcal{S}$ denote the Borel \sigmafield/ on $\mathbb{S}$ generated by the open sets under $\rho$. Together $(\mathbb{S},\mathcal{S})$ defines a measurable space.

Similarly, $\mathbb{S}^n$ is equipped with the metric $\rho_n(x,y)=[\sum_{i=1}^{n}\rho^2(x_i,y_i)]^{\frac{1}{2}}$, and the corresponding Borel \sigmafield/ under $\rho_n$ is denoted by $\mathcal{S}^{\prime n}$. Together $(\mathbb{S}^n,\mathcal{S}^{\prime n})$ is the measurable space for $n$ tuples.

Next, consider the space of infinite sequences $\mathbb{S}^{\infty}=\{(x_1,x_2,\dots) \, | \, x_i\in \mathbb{S},i\in {\mathbb{N}_+} \}$.
It is equipped with the metric 
\[\rho_{\infty}(x,y)=\sum\limits_{i=1}^{\infty} \frac{1}{2^i} \cdot \frac{\rho (x_i,y_i)}{1+\rho (x_i,y_i)}.\]
The Borel \sigmafield/ on $\mathbb{S}^\infty$ under $\rho_\infty$ is denoted by $\mathcal{S}^{\prime \infty}$. Then $(\mathbb{S}^n,\mathcal{S}^{\prime \infty})$ is the measurable space for infinite sequences.

Since $\mathbb{S}$ is separable and complete, it can be proved that $\mathbb{S}^n$ and $\mathbb{S}^\infty$ are also separable and complete (Appendix M6 in \cite{conv}). In addition, because of separability, the Borel \sigmafield/s $\mathcal{S}^{\prime n}$ and $\mathcal{S}^{\prime \infty}$ are equal to $\mathcal{S}^{n}$ and $\mathcal{S}^{\infty}$, respectively. In other words, the Borel \sigmafield/s $\mathcal{S}^{\prime n}$ and $\mathcal{S}^{\prime \infty}$ generated by the collection of open sets under the corresponding metrics coincide with the product \sigmafield/s generated by all measurable rectangles ($\mathcal{S}^n$) and all measurable cylinder sets ($\mathcal{S}^\infty$), respectively (Lemma 1.2 in \cite{foundation}). Therefore, from now on we write $\mathcal{S}^{n}$ and $\mathcal{S}^{\infty}$ for the corresponding Borel \sigmafield/s. Finally, let $\mathbb{M}$, $\mathbb{M}^n$ and $\mathbb{M^\infty}$ denote the set of all random elements of $\mathbb{S}$, $\mathbb{S}^n$ and $\mathbb{S}^\infty$, respectively.

Let $\pi_n: \mathbb{S}^\infty \to \mathbb{S}^n$ be the natural projection: $\pi_n(x)=(x_1,x_2,\dots,x_n)$. Since given $\mb{x}\in \mathbb{M}^\infty$, $(\pi_n \circ \mb{x}):\Omega \to \mathbb{S}^n$ defines a random element of $\mathbb{S}^n$ projected from $\mathbb{S}^\infty$, we also use $\pi_n$ to denote the mapping: $\pi_n: \mathbb{M}^\infty \to \mathbb{M}^n$ where $\pi_n(\mb{x})=(\mb{x}_1,\mb{x}_2,\dots,\mb{x}_n)$. By definition, $\pi_n$ is the operator which truncates random sequences to random vectors. Given $A\subset \mathbb{S}^\infty$, we use $\pi_n(A)$ to denote the projection of $A$, i.e. $\pi_n(A)=\{x\in \mathbb{S}^n:x=\pi_n(y)\textnormal{ for some } y\in A\}$. 

\subsection{Analytical Framework for EA and IPMs}
\label{chapter:framework-model}

In this section, we present an analytical framework for the EA and IPMs. First, the modeling assumptions are stated. We only deal with operators which generate c.i.i.d. individuals. Then, we present an abstraction of the EA and IPMs. This abstraction serves as the basis for building our framework. Finally, the framework is presented. It unifies the range spaces of the random elements and defines the convergence of IPMs.

\subsubsection{Modeling Assumptions}
\label{chapter:framework-model-assumption}

We assume that the EA on the problem (\ref{eqn:framework-obj}) is time homogeneous and Markovian, such that the next generation depends only on the current one, and the transition rule from the $k$th generation to the $(k+1)$th generation is invariant with respect to $k\in \mathbb{N}$. We further assume that individuals in the next generation are c.i.i.d. given the current generation. As this assumption is the only extra assumption introduced in the framework, it may need some further explanation.

The main reason for introducing this assumption is to simplify the analysis. Conditional independence implies exchangeability, therefore individuals in the same generation $k\in \mathbb{N}_+$ are always exchangeable. As a result, it is possible to exploit the symmetry in the population and study the transition equations of marginal distributions. Besides, it is because of conditional independence that we can easily expand the random elements modeling finite-sized populations to random sequences, and therefore define convergence in distribution for random elements of the corresponding metric space. In addition, many real world operators in EAs satisfy this assumption, such as the proportionate selection operator and the crossover operator analyzed in \cite{Qi1,Qi2}.

However, we admit that there are some exceptions to our assumption. A most notable one may be the mutation operator, though it does not pose significant difficulties. The mutation operator perturbs each individual in the current population independently, according to a common conditional p.d.f. If the current population is not exchangeable, then after mutation the resultant population is not exchangeable, either. Therefore, it seems that mutation does not produce c.i.i.d. individuals. However, considering the fact that mutation is often used along with other operators, as long as these other operators generate c.i.i.d. populations, the individuals after mutation will be c.i.i.d., too. Therefore, a combined operator of mutation and any other operator satisfying the c.i.i.d. assumption can satisfy our assumption. An example can be seen in \cite{Qi1}, where mutation is analyzed together with proportionate selection. On the other hand, an algorithm which only uses mutation is very simple. It can be readily modeled and analyzed without much difficulty. 
 
Perhaps more significant exceptions are operators such as selection \emph{without} replacement, or the crossover operator which produces two dependent offspring at the same time. In fact, for these operators not satisfying the c.i.i.d. assumption, it is still possible to expand the random elements modeling finite-sized population to random sequences. For example, the random elements can be padded with some fixed constants or random elements of known distributions to form the random sequences. In this way, our definition of the convergence of IPMs can still be applied. However, whether in this scenario convergence in distribution for these random sequences can still yield meaningful results similar to the transition equation is another research problem. It may need further investigation. Nonetheless, our assumption is equivalent to the exchangeability assumption generally used in previous studies. 

\subsubsection{The Abstraction of EA and IPMs}
\label{chapter:framework-model-abstraction}

Given the modeling assumptions, we develop an abstraction to describe the population dynamics of the EA and IPMs. 

Let the EA with population size $n$ be denoted by $\EA{n}$, and the $k$th $(k\in \mathbb{N})$ generation it produces be modeled as a random element $\Pkn=(\Xkin{k}{i}{n} )_{i=1}^n\in \mathbb{M}^n$, where $\Xkin{k}{i}{n}\in \mathbb{M}$ is a random element representing the $i$th individual in $\Pkn$. Without loss of generality, assume that the EA has two operators, $\operatorname{G}$ and $\operatorname{H}$. In each iteration, the EA first employs $\operatorname{G}$ on the current population to generate an intermediate population, on which it then employs $\operatorname{H}$ to generate the next population. Notice that here $\operatorname{G}$ and $\operatorname{H}$ are just terms representing the operators in the real EA. They facilitate describing the evolutionary process. For $\EA{n}$, $\operatorname{G}$ and $\operatorname{H}$ are actually \emph{instantiated} as functions from $\mathbb{M}^n$ to $\mathbb{M}^n$, denoted by $\operatorname{G}_n$ and $\operatorname{H}_n$, respectively. For example, if $\operatorname{G}$ represents proportionate selection, the function $\operatorname{G}_n:\mathbb{M}^n \to \mathbb{M}^n$ is the actual operator in $\EA{n}$ generating $n$ c.i.i.d. individuals according to the conditional probability (\ref{eqn:Qi-sel}). Of course, for the above abstraction to be valid, the operators used in $\EA{n}$ should actually produce random elements in $\mathbb{M}^n$, i.e. the newly generated population should be measurable on $(\mathbb{S}^n,\mathcal{S}^n)$. As most operators in real EAs satisfy this condition and this is the assumption implicitly taken in previous studies, we assume that this condition is automatically satisfied.

Given these notations, the evolutionary process of $\EA{n}$ can be described by the sequence $(\PP{k}{n})_{k=0}^{\infty}$, where the initial population $\PP{0}{n}$ is known and the generation of $\Pkn,k\in \mathbb{N}_+$ follows the recurrence equation
\begin{equation}
\label{eqn:framework-recurPkn}
\PP{k+1}{n}=(\operatorname{H}_n \circ \operatorname{G}_n) (\Pkn).
\end{equation}
Then understanding the population dynamics of the EA can be achieved by studying the distributions and properties of $\Pkn$.

Let the IPM of the EA be denoted by $\EA{\infty}$. The population dynamics it produces can be described by the sequence $(\Pkinf\in \mathbb{M}^\infty)_{k=0}^{\infty}$, where $\PP{0}{\infty}$ is known and the generation of $\Pkinf,k\in \mathbb{N}_+$ follows the recurrence equation
\begin{equation}
\label{eqn:framework-recurPkinf}
\PP{k+1}{\infty}=(\operatorname{H}_\infty \circ \operatorname{G}_\infty) (\Pkinf),
\end{equation}  
in which $\operatorname{G}_{\infty},\operatorname{H}_{\infty}:\mathbb{M}^{\infty}\to \mathbb{M}^{\infty}$ are operators in $\EA{\infty}$ modeled after $\operatorname{G}$ and $\operatorname{H}$. Then, the convergence of $\EA{\infty}$ basically requires that $(\Pkn)_{n=1}^{\infty}$ converges to $\Pkinf$ for every generation $k$.

\subsubsection{The Proposed Framework}
\label{chapter:framework-model-model}

As stated before, for each generation $k\in \mathbb{N}$, the elements of the sequence $(\PP{k}{1},\PP{k}{2},\dots)$ and the limit $\Pkinf$ are all random elements of different metric spaces. Therefore, the core of developing our model is to expand $\Pkn$ to random sequences, while ensuring that this expansion will not affect modeling the evolutionary process of the real EA. The result of this step is the sequence of random sequences $(\PP[Q]{k}{n}\in \mathbb{M}^\infty )_{k=0}^\infty$ for each $n\in \mathbb{N}_+$, which completely describes the population dynamics of $\EA{n}$. For the population dynamics of $\EA{\infty}$, we just let $\Qkinf=\Pkinf$.

The expansion of $\Pkn$ and the relationships between $\Pkn$, $\Qkn$ and $\Qkinf$ are the core of our framework. In the following, we present them rigorously.

\subsubsection{The Expansion of $\Pkn$}

We start by decomposing each of $\operatorname{G}_n$ and $\operatorname{H}_n$ to two operators. One operator is from $\mathbb{S}^\infty$ to $\mathbb{S}^n$. It corresponds to how to convert random sequences to random vectors. A natural choice is the projection operator $\pi_n$.

To model the evolutionary process, we also have to define how to expand random vectors to random sequences. In other words, we have to define the expansions of $\operatorname{G_n}$ and $\operatorname{H_n}$, which are functions from $\mathbb{S}^n$ to $\mathbb{S}^\infty$.

\begin{definition}[The expansion of operator]
	\label{def:framework-expansion}
	For an operator $\operatorname{T}_n:\mathbb{M}^n \to \mathbb{M}^n$ satisfying the condition that for any $\mb{x}\in \mathbb{M}^n$, the elements of $\operatorname{T}_n(\mb{x})$ are c.i.i.d. given $\mb{x}$, the expansion of $\operatorname{T}_n$ is the operator $\widetilde{\operatorname{T}}_n:\mathbb{M}^n \to \mathbb{M}^\infty$, satisfying that for any $\mb{x}\in \mathbb{M}^n$,
	\begin{enumerate}
		\item $\operatorname{T}_n(\mb{x}) = (\pi_n \circ \widetilde{\operatorname{T}}_n)(\mb{x})$. \label{def:framework-expansion-condi1}
		\item The elements of $\widetilde{\operatorname{T}}_n(\mb{x})$ are c.i.i.d. given $\mb{x}$. \label{def:framework-expansion-condi2}
	\end{enumerate}
\end{definition}

In Definition \ref{def:framework-expansion}, the operator $\widetilde{\operatorname{T}}_n$ is the expansion of $\operatorname{T}_n$. Condition \ref{def:framework-expansion-condi1}) ensures that $\operatorname{T}_n$ can be safely replaced by $\pi_n \circ \widetilde{\operatorname{T}}_n$. Condition \ref{def:framework-expansion-condi2}) ensures that the paddings for the sequence are generated according to the same conditional probability distribution as that used by $\operatorname{T}_n$ to generate new individuals. In other words, if the operator $\dot{\operatorname{T}}_n:\mathbb{M}^n \to \mathbb{M}$ describes how $\operatorname{T}_n$ generates each new individual from the current population, $\operatorname{T}_n$ is equivalent to invoking $\dot{\operatorname{T}}_n$ independently on the current population for $n$ times, and $\widetilde{\operatorname{T}}_n$ is equivalent to invoking $\dot{\operatorname{T}}_n$ independently for infinite times. Finally, because $\operatorname{T}_n$ satisfies the condition in the premise, the expansion $\widetilde{\operatorname{T}}_n$ always exists.

By Definition \ref{def:framework-expansion}, the operators in $\EA{n}$ can be decomposed as $\operatorname{G}_n=\pi_n \circ \widetilde{\operatorname{G}}_n$ and $\operatorname{H}_n=\pi_n \circ \widetilde{\operatorname{H}}_n$, respectively. Then, the evolutionary process of $\EA{n}$ can be described by the sequence of random sequences $[\Qkn=(\Xkin[y]{k}{i}{n})_{i=0}^\infty\in \mathbb{M}^\infty]_{k=0}^\infty$, satisfying the recurrence equation
\begin{equation}
\label{eqn:framework-qknpkn}
\PP[Q]{k+1}{n}=(\widetilde{\operatorname{H}}_n \circ \pi_n \circ \widetilde{\operatorname{G}}_n)(\Pkn),
\end{equation}
where $\Pkn$ follows the recurrence equation (\ref{eqn:framework-recurPkn}), and $\PP[Q]{0}{n}=(\PP{0}{n},0,0,\dots)$. It can also be proved that
\begin{equation}
\label{eqn:framework-pknqkn}
\Pkn=\pi_n(\Qkn).
\end{equation}

Essentially, (\ref{eqn:framework-qknpkn}) and (\ref{eqn:framework-pknqkn}) describe how the algorithm progresses in the order $\dots,\Qkn,\Pkn,\PP[Q]{k+1}{n},\PP{k+1}{n},\dots$. It fully characterizes the population dynamics $(\Pkn)_k$, and it is clear that the extra step of generating $\Qkn$ does not introduce modeling errors.

For $\EA{\infty}$, because $\Pkinf\in \mathbb{M}^\infty$, there is no need for expansion. For convenience we simply let 
\begin{equation}
\label{eqn:framework-qkninf}
\Qkinf=\Pkinf
\end{equation}
for $k\in \mathbb{N}$.

In summary, the relationships between $\Pkn$, $\Qkn$ and $\Qkinf$ are better illustrated in Fig. \ref{fig:framework-pknqknqkinf}. This is the core of our framework for modeling the EA and IPMs. For clarity, we also show the intermediate populations generated by $\operatorname{G}$ (denoted by $\PP[P^\prime]{k}{n}$), their expansions (denoted by $\PP[Q^\prime]{k}{n}$), and their counterparts generated by $\operatorname{G}_\infty$ (denoted by $\PP[Q^\prime]{k}{\infty}$), respectively. How they fit in the evolutionary process can be clearly seen in the figure.

In Fig. \ref{fig:framework-pknqknqkinf}, a solid arrow with an operator on it means that the item at the arrow head equals the result of applying the operator on the item at the arrow tail. For example, from the figure it can be read that $\PP[Q]{1}{n}=\widetilde{\operatorname{H}}_n(\PP[P^\prime]{0}{n})$. Dashed arrow with a question mark on it signals the place to check whether convergence in distribution holds. For example, when $k=2$, it should be checked whether $({\PP[Q]{2}{n}})_{n=1}^{\infty}$ converges to $\PP[Q]{2}{\infty}$ as $n\to \infty$.

\begin{figure*}[ht]
	\centering
	\[
	\xymatrix
	{
		\PP{0}{n} \ar[r]^{\widetilde{\operatorname{G}}_n}& \bullet \ar[r]^{\pi_n} \ar[d] & \PP[P^\prime]{0}{n} \ar[r]^{\widetilde{\operatorname{H}}_n}& \bullet \ar[r]^{\pi_n} \ar[d] & \PP{1}{n} \ar[r]^{\widetilde{\operatorname{G}}_n} & \bullet \ar[r]^{\pi_n} \ar[d] & \PP[P^\prime]{1}{n} \ar[r]^{\widetilde{\operatorname{H}}_n} & \bullet \ar[r] \ar[d] & \\	
		\PP[Q]{0}{n} \ar[u]^{\pi_n} \ar@{-->}[d]^{?} & \PP[Q^\prime]{0}{n}\ar@{-->}[d]^{?} &  & \PP[Q]{1}{n} \ar@{-->}[d]^{?}&  &  \PP[Q^\prime]{1}{n}\ar@{-->}[d]^{?} & & \PP[Q]{2}{n} \ar@{-->}[d]^{?}&	\\
		\PP[Q]{0}{\infty} \ar[r]^{\operatorname{G}_\infty}& \PP[Q^\prime]{0}{\infty} \ar[rr]^{\operatorname{H}_\infty}&  & \PP[Q]{1}{\infty} \ar[rr]^{\operatorname{G}_\infty}&  &  \PP[Q^\prime]{1}{\infty} \ar[rr]^{\operatorname{H}_\infty}& & \PP[Q]{2}{\infty} \ar[r]&
	}
	\]
	\caption{Relationships between $\Pkn$, $\Qkn$ and $\Qkinf$}
	\label{fig:framework-pknqknqkinf}
\end{figure*}

Finally, one distinction needs special notice. For $\EA{m}$ and $\EA{n}$ ($m\neq n$), consider the operators to generate $\PP{k}{m}$ and $\Pkn$. It is clear that $\operatorname{G}_m:\mathbb{M}^m \to \mathbb{M}^m$ and $\operatorname{G}_n:\mathbb{M}^n \to \mathbb{M}^n$ are two different operators because their domains and ranges are all different. The distinction still exists when we consider $\Qkn$, though it is more subtle and likely to be ignored. In Fig. \ref{fig:framework-pknqknqkinf}, if we consider the operator $\widehat{\operatorname{G}}_n=\pi_n \circ \widetilde{\operatorname{G}}_n:\mathbb{M}^\infty \to \mathbb{M}^\infty$, it is clear that $\widehat{\operatorname{G}}_n$ uses the same mechanism to generate new individuals as the one used in $\operatorname{G}_n=\widetilde{\operatorname{G}}_n\circ \pi_n$, and $\PP[Q^\prime]{k}{n}=\widehat{\operatorname{G}}_n(\Qkn)$ describes the same population dynamics as that generated by $\PP[P^\prime]{k}{n}=\operatorname{G}_n(\Pkn)$. However, if we choose $m\neq n$, $\widehat{\operatorname{G}}_m$ and $\widehat{\operatorname{G}}_n$ are both functions from $\mathbb{M}^\infty$ to $\mathbb{M}^\infty$. Therefore, checking domains and ranges are not enough to discern $\widehat{\operatorname{G}}_m$ and $\widehat{\operatorname{G}}_n$. It is important to realize that the distinction between $\widehat{\operatorname{G}}_m$ and $\widehat{\operatorname{G}}_n$ lies in the \emph{contents} of the functions. $\widehat{\operatorname{G}}_m$ and $\widehat{\operatorname{G}}_n$ use $m$ and $n$ individuals in the current population to generate the new population, respectively, although the new population contains infinite number of individuals. In short, $\EA{m}$ and $\EA{n}$ are the EA instantiated with different population sizes. Mathematically, the corresponding population dynamics are modeled by stochastic processes involving \emph{different} operators, even though their domains and ranges may be the same. The same conclusion also holds for the operator $\operatorname{H}$.

\subsection{Convergence of IPMs}
\label{chapter:framework-conv}

Given the framework modeling the EA and IPMs, first, we define convergence in distribution for random elements of $\mathbb{S}^\infty$. This is standard material. Then, the convergence of IPMs is defined by requiring that the sequence $(\PP[Q]{k}{1},\PP[Q]{k}{2},\dots)$ converges to $\Qkinf$ for every $k\in \mathbb{N}$.

\subsubsection{Convergence in Distribution}

As $\Qkn$ are random elements of $\mathbb{S}^\infty$, in the following we define convergence in distribution for sequences of $\mathbb{S}^\infty$-valued random elements. Convergence in distribution is equivalent to weak convergence of \emph{induced} probability measures of the random elements. We use the former theory because when modeling individuals and populations as random elements, the former theory is more intuitive and straightforward. The following materials are standard. They contain the definition of convergence in distribution for random elements, as well as some useful definitions and theorems which are used in our analysis of the simple EA. Most of the materials are collected from the theorems and examples in Sections 1-3 of \cite{conv}. The definition of Prokhorov metric is collected from Section 11.3 in \cite{dudley}. 

Let $\mb{x},\mb{y},\mb{x}_n,n\in {\mathbb{N}_+}$ be random elements defined on a hidden probability space $(\Omega,\mathcal{F},\Prob)$ taking values in some separable metric space $\mathbb{T}$. $\mathbb{T}$ is coupled with the Borel \sigmafield/ $\mathcal{T}$. Let $(\mathbb{T}^\prime,\mathcal{T}^\prime)$ be a separable measurable space other than $(\mathbb{T},\mathcal{T})$.

\begin{definition}[Convergence in distribution]
	\label{def:framework-cid}
	If the sequence $(\mb{x}_n)_{n=1}^\infty$ satisfies the condition that $\EXP\left[h(\mb{x}_n)\right]\to \EXP\left[h(\mb{x})\right]$ for every bounded, continuous function $h:\mathbb{T}\to \mathbb{R}$, we say $(\mb{x}_n)_{n=1}^\infty$ converges in distribution to $\mb{x}$, and write $\mb{x}_n \tod \mb{x}$.
\end{definition}

For $\epsilon>0$, let $A^\epsilon=\{y\in \mathbb{T}:d(x,y)<\epsilon \textnormal{ for some } x\in A\}$.
Then it is well known that convergence in distribution on separable metric spaces can be metricized by the Prokhorov metric.

\begin{definition}[Prokhorov metric]
	\label{def:prokhorov}
	For two random elements $\mb{x}$ and $\mb{y}$, the Prokhorov metric is defined as 
	\[\dist{d}(\mb{x},\mb{y})=\inf \{\epsilon>0:\Prob(\mb{x}\in A)\leq \Prob(\mb{y}\in A^\epsilon)+\epsilon,\forall A\in \mathcal{T} \}.\]
\end{definition}

Call a set $A$ in $\mathcal{T}$ an $\mb{x}$-continuity set if $\textnormal{P}(\mb{x}\in \partial A)=0$, where $\partial A$ is the boundary set of $A$.

\begin{theorem}[The Portmanteau theorem]
	\label{thm:framework-portmanteau}
	The following statements are equivalent.
	\begin{enumerate}
		\item $\mb{x}_n \tod \mb{x}$.
		\item $\limsup_n \Prob(\mb{x}_n \in F)\leq \Prob(\mb{x}\in F)$ for all closed set $F\in \mathcal{T}$.
		\item $\liminf_n \Prob(\mb{x}_n \in G)\geq \Prob(\mb{x}\in G)$ for all open $G\in \mathcal{T}$.
		\item $\Prob(\mb{x}_n \in A)\to \Prob(\mb{x}\in A)$ for all $\mb{x}$-continuity set $A\in \mathcal{T}$. \label{thm:framework-portmanteau-condi4}
	\end{enumerate}
\end{theorem}

\begin{theorem}[The mapping theorem]
	\label{thm:framework-contimapping}
	Suppose $h:(\mathbb{T},\mathcal{T})\to (\mathbb{T}^\prime,\mathcal{T}^\prime)$ is a measurable function. Denote by $D_h$ the set of discontinuities of $h$. If $\mb{x}_n \tod \mb{x}$ and $\textnormal{P}(D_h)=0$, then $h(\mb{x}_n) \tod h(\mb{x})$.
\end{theorem}

Let $\mb{a},\mb{a}_n$ be random elements of $\mathbb{T}$, $\mb{b},\mb{b}_n$ be random elements of $\mathbb{T}^\prime$, then 
$(\mb{a} \; \mb{b})^T$ and $(\mb{a}_n \; \mb{b}_n)^T$ are random elements of $\mathbb{T}\times \mathbb{T}^\prime$. Note that $\mathbb{T}\times \mathbb{T}^\prime$ is separable.

\begin{theorem}[Convergence in distribution for product spaces]
	\label{thm:framework-product}
If $\mb{a}$ is independent of $\mb{b}$ and $\mb{a}_n$ is independent of $\mb{b}_n$ for all $n\in \mathbb{N}_+$, then $(\mb{a}_n \; \mb{b}_n)^T\tod (\mb{a} \; \mb{b})^T$ if and only if $\mb{a}_n\tod \mb{a}$ and $\mb{b}_n \tod \mb{b}$.
\end{theorem}

Theorem \ref{thm:framework-product} is adapted from Theorem 2.8 (ii) in \cite{conv}.

Let $\mb{z},\mb{z}_n,n\in {\mathbb{N}_+}$ be random elements of $\mathbb{S}^\infty$.

\begin{theorem}[Finite-dimensional convergence]
	\label{thm:framework-finiteconv}
	$\mb{z}_n\tod \mb{z}$ if and only if $\pi_m(\mb{z}_n)\tod \pi_m(\mb{z})$ for any $m\in \mathbb{N}_+$.
\end{theorem}
Theorem \ref{thm:framework-finiteconv} basically asserts that convergence in distribution for countably infinite dimensional random elements can be studied through their finite-dimensional projections. It is adapted from Example 1.2 and Example 2.4 in \cite{conv}. In \cite{conv}, the metric space under consideration is $\mathbb{R}^\infty$. However, as both $\mathbb{R}$ and $\mathbb{S}$ are separable, it is not difficult to adapt the proofs for $\mathbb{R}^\infty$ to a proof for Theorem \ref{thm:framework-finiteconv}. Note that $\pi_m(\mb{z})$ are random elements defined on $(\Omega,\mathcal{F},\Prob)$ taking values in $(\mathbb{S}^m,\mathcal{S}^m)$, and $\Prob[\pi_m(\mb{z}) \in A]=\Prob(\mb{z}\in A\times \mathbb{S}\times \mathbb{S}\times \dots)$ for every $A\in \mathcal{S}^m$. The same is true for $\pi_m(\mb{z}_n)$.

\subsubsection{Convergence of IPM}

As convergence in distribution is properly defined, we can use the theory to define convergence of IPMs. The idea is that IPM is convergent (thus justifiable) if and only if it can predict the limit distribution of the population dynamics of $\EA{n}$ for \emph{every} generation $k\in \mathbb{N}$ as the population size $n$ goes to infinity. It captures the limiting behaviors of real EAs.

\begin{definition}[Convergence of IPMs]
	\label{def:framework-convipm}
	An infinite population model $\EA{\infty}$ is convergent if and only if for every $k\in \mathbb{N}$, $\Qkn \tod \Qkinf$ as $n\to \infty$, where $\Qkn$, $\Qkinf$ and the underling $\Pkn$, $\Pkinf$ are generated according to (\ref{eqn:framework-qknpkn}), (\ref{eqn:framework-qkninf}), (\ref{eqn:framework-recurPkn}) and (\ref{eqn:framework-recurPkinf}).
\end{definition}

Definition \ref{def:framework-convipm} is essentially the core of our proposed framework. It defines the convergence of IPM and is rigorous and clear.

\subsection{Summary}
\label{chapter:framework-summary}

In this section, we built a framework to analyze the convergence of IPMs. The most significant feature of the framework is that we model the populations as random sequences, thereby unifying the ranges of the random elements in a common metric space. Then, we gave a rigorous definition for the convergence of IPMs based on the theory of convergence in distribution.

Our framework is general. It only requires that operators produce c.i.i.d. individuals. In fact, any EA and IPM satisfying this assumption can be put into the framework. However, to obtain meaningful results, the convergence of IPMs has to be proved. This may require extra analyses on IPM and the inner mechanisms of the operators. These analyses are presented in Section \ref{analysis}.

Finally, there is one thing worth discussing. In our framework, the expansion of operator is carried out by padding the finite population with c.i.i.d. individuals following the \emph{same} marginal distribution. Then a question naturally arises: why not pad the finite population with some other random elements, or just with the constant $0$? This idea deserves consideration. After all, if the expansion is conducted by padding $0$s, the requirement of c.i.i.d. can be discarded, and the framework and the convergence of IPMs stay the same. However, we did not choose this approach. The reason is that padding the population with c.i.i.d. individuals facilities analysis of the IPM. For example, in our analysis in Section \ref{analysis}, the sufficient conditions for the convergence of IPMs require us to consider $\operatorname{\Gamma}_m(\Qkn)$, where $\operatorname{\Gamma}$ is the operator under analysis. $\operatorname{\Gamma}_m$ uses the first $m$ elements of $\Qkn$ to generate new individuals. Now if $m>n$ and $\Qkn$ is expanded from $\Pkn$ by padding $0$s, $\operatorname{\Gamma}_m(\Qkn)$ does not make any sense because the $m$ individuals used by $\operatorname{\Gamma}_m$ have $(m-n)$ $0$s. This restricts our option in proving the convergence of IPMs.

\section{Analysis of the Simple EA}
\label{analysis}

In this section, we analyze the simple EA using our framework. In Section \ref{chapter:analy-sufficient}, we give sufficient conditions for the convergence of IPMs. To appreciate the necessity, consider the framework in Fig. \ref{fig:framework-pknqknqkinf}. To prove the convergence of IPM, by Definition \ref{def:framework-convipm}, we should check whether $\Qkn \tod \PP[Q]{k}{\infty}$ as $n\to \infty$ for every $k\in \mathbb{N}$. However, this direct approach is usually not viable. To manually check the convergence for all values of $k$ is wearisome and sometimes difficult. This is because as $k$ increases, the distributions of $\Qkn$ and $\Qkinf$ change. Therefore, the method needed to prove $\Qkn\tod \PP[Q]{k}{\infty}$ as $n\to \infty$ may be different from the method needed to prove $\PP[Q]{k+1}{n} \tod \PP[Q]{k+1}{\infty}$ as $n\to \infty$. Of course, after proving the cases for several values of $k$, it may be possible to discover some patterns in the proofs, which can be extended to cover other values of $k$, thus proving the convergence of the IPM. But this process is still tedious and uncertain.

In view of this, a ``smarter'' way to prove the convergence of IPM may be the following method. First, the convergence of IPM for \emph{one} iteration step for \emph{each} operator is proved. Then, the results are combined and extended to cover the whole population dynamics. The idea is that if the convergence holds for one generation number $k$, then it can be passed on automatically to all subsequent generations. For example, in Fig. \ref{fig:framework-pknqknqkinf}, consider the operators $\operatorname{G}_\infty$ and $\widetilde{\operatorname{G}}_n \circ \pi_n$. The first step is to prove that 
\begin{equation}
	\textnormal{if } \Qkn \tod \PP[Q]{k}{\infty} \textnormal{ as } n\to \infty, \textnormal{then } \PP[Q]{k}{\prime n} \tod \PP[Q]{k}{\prime \infty} \textnormal{ as } n\to \infty.\label{statement:analy-G}
\end{equation}
In other words, $\operatorname{G}_\infty$ can model $\widetilde{\operatorname{G}}_n \circ \pi_n$ for one iteration step. Then, after obtaining similar results for $\operatorname{H}_\infty$ and $\widetilde{\operatorname{H}}_n \circ \pi_n$, we combine the results together and the convergence of the overall IPM is proved.

However, this approach still seems difficult because we have to prove this pass-on relation (\ref{statement:analy-G}) holds for every $k$. In essence, this corresponds to whether the operators in IPM can be stacked together and iterated for any number of steps. This is the issue of the stacking of operators and iterating the algorithm. Therefore, in Section \ref{chapter:analy-sufficient}, we give sufficient conditions for this to hold. These conditions are important. If they hold, proving the convergence of the overall IPM can be broken down to proving the convergence of one iteration step of each operator in IPM. This greatly reduces the difficulty in deriving the proof.

To model real EAs, IPM has to be constructed reasonably. As shown in Section \ref{qiwrong}, exchangeability cannot yield the transition equation for the simple EA. This creates the research problem of finding a suitable modeling assumption to derive IPM. Therefore, in Section \ref{chapter:analy-iid}, we discuss the issue and propose to use i.i.d. as the modeling assumption in IPM.

Then, we use the sufficient conditions to prove the convergence of IPMs for various operators. The operators of mutation and $k$-ary recombination are readily analyzed in Section \ref{chapter:analy-mut} and Section \ref{chapter:analy-recomb}, respectively. In Section \ref{chapter:analy-sum}, we summarize this section and discuss our results.

\subsection{Sufficient Conditions for Convergence of IPMs}
\label{chapter:analy-sufficient}

To derive sufficient conditions for the convergence of the overall IPM, the core step is to derive conditions under which the operators in the IPM can be stacked and iterated.

As before, let $\EA{n}$ and $\EA{\infty}$ denote the EA with population size $n$ and the IPM under analysis, respectively. Let $\operatorname{\Gamma}$ be an operator in the EA, and $\operatorname{\Gamma}_n:\mathbb{M}^\infty \to \mathbb{M}^\infty$ and $\operatorname{\Gamma}_\infty:\mathbb{M}^\infty \to \mathbb{M}^\infty$ be its corresponding expanded operators in $\EA{n}$ and $\EA{\infty}$, respectively. Note that $\operatorname{\Gamma}_n$ and $\operatorname{\Gamma}_\infty$ generate random elements of $\mathbb{S}^\infty$. To give an example, $\operatorname{\Gamma}_n$ and $\operatorname{\Gamma}_\infty$ may correspond to $\pi_n \circ \widetilde{\operatorname{G}}_n$ and $\operatorname{G}_\infty$ in Fig. \ref{fig:framework-pknqknqkinf}, respectively.

We define a property under which $\operatorname{\Gamma}_\infty$ can be stacked with some other operator $\operatorname{\Psi}_\infty$ satisfying the same property without affecting the convergence of the overall IPM. In other words, for an EA using $\operatorname{\Psi}$ and $\operatorname{\Gamma}$ as its operators, we can prove the convergence of IPM by studying $\operatorname{\Psi}$ and $\operatorname{\Gamma}$ separately. We call this property ``the stacking property''. It is worth noting that if $\operatorname{\Phi}=\operatorname{\Gamma}$, then this property guarantees that $\operatorname{\Gamma}_\infty$ can be iterated for any number of times. Therefore it also resolves the issue of iterating the algorithm.

Let $\mb{A}_\alpha$ be random elements in $\mathbb{M}^\infty$ for $\alpha\in\mathbb{N}_+\cup \{\infty\}$. We have the following results.

\begin{definition}[The stacking property]
	\label{def:analy-stacking}
Given $\mathbb{U} \subset \mathbb{M}^\infty$, if for any converging sequence $\mb{A}_n \tod \mb{A}_\infty \in \mathbb{U}$, $\operatorname{\Gamma}_n(\mb{A}_n) \tod \operatorname{\Gamma}_\infty(\mb{A}_\infty)\in \mathbb{U}$ as $n\to \infty$ always holds, then we say that $\operatorname{\Gamma}_\infty$ has the stacking property on $\mathbb{U}$.
\end{definition}

\begin{theorem}
	\label{thm:analy-stacking}
	If $\operatorname{\Psi}_\infty$ and $\operatorname{\Gamma}_\infty$ have the stacking property on $\mathbb{U}$, then $\operatorname{\Psi}_\infty \circ \operatorname{\Gamma}_\infty$ has the stacking property on $\mathbb{U}$.
\end{theorem}

\begin{proof}
	For any converging sequence $\mb{A}_n\tod \mb{A}_\infty\in \mathbb{U}\subset \mathbb{M}^\infty$, because $\operatorname{\Gamma}_\infty$ has the stacking property on $\mathbb{U}$, we have $\operatorname{\Gamma}_n(\mb{A}_n) \tod \operatorname{\Gamma}_\infty(\mb{A}_\infty)\in \mathbb{U}$. Then, $(\operatorname{\Gamma}_n(\mb{A}_n))_n$ is also a converging sequence. Since $\operatorname{\Psi}_\infty$ has the stacking property on $\mathbb{U}$, then by definition we immediately have $(\operatorname{\Psi}_n \circ \operatorname{\Gamma}_n)(\mb{A}_n) \tod (\operatorname{\Psi}_\infty \circ \operatorname{\Gamma}_\infty)(\mb{A}_\infty)\in \mathbb{U}$.
\end{proof}

By Theorem \ref{thm:analy-stacking}, any composition of $\operatorname{\Psi}_\infty$ and $\operatorname{\Gamma}_\infty$ has the stacking property on $\mathbb{U}$. In particular, $(\operatorname{\Gamma}_\infty)^m$ has the stacking property on $\mathbb{U}$. The stacking property essentially guarantees that the convergence on $\mathbb{U}$ can be passed on to subsequent generations.

\begin{theorem}[Sufficient condition 1]
	\label{thm:analy-sufficient1}
	For an EA consisting of a single operator $\operatorname{\Gamma}$, let $\operatorname{\Gamma}$ be modeled by $\operatorname{\Gamma}_\infty$ in the IPM $\EA{\infty}$ and $\operatorname{\Gamma}_\infty$ have the stacking property on some space $\mathbb{U}\subset \mathbb{M}^\infty$. If the initial populations of both EA and $\EA{\infty}$ follow the same distribution $\Prob_{\mb{X}}$ for some $\mb{X}\in\mathbb{U}$, then $\EA{\infty}$ converges.
\end{theorem}

\begin{proof}
	Note that for $\EA{n}$ and $\EA{\infty}$, the $k$th populations they generate are $(\operatorname{\Gamma}_n)^k(\mb{X})$ and $(\operatorname{\Gamma}_\infty)^k(\mb{X})$, respectively. By Theorem \ref{thm:analy-stacking}, $(\operatorname{\Gamma}_\infty)^k$ has the stacking property on $\mathbb{U}$. Because the sequence $(\mb{X}, \mb{X},\dots)$ converges to $\mb{X}\in \mathbb{U}$, by Definition \ref{def:analy-stacking}, $(\operatorname{\Gamma}_n)^k(\mb{X})\tod (\operatorname{\Gamma}_\infty)^k(\mb{X})\in \mathbb{U}$ as $n\to \infty$. Since this holds for any $k\in \mathbb{N}$, by Definition \ref{def:framework-convipm}, $\EA{\infty}$ converges.
\end{proof}

By Theorem \ref{thm:analy-stacking} and Theorem \ref{thm:analy-sufficient1}, we can prove the convergence of the overall IPM by proving that the operators in the IPM have the stacking property. Comparing with (\ref{statement:analy-G}), it is clear that the stacking property is a sufficient condition. This is because the stacking property requires that $(\operatorname{\Gamma}_n(\mb{A}_n))_n$ converges to a point in $\mathbb{U}$ for \emph{any} converging sequence $(\mb{A}_n)_n$ satisfying $(\mb{A}_n)_n\tod \mb{A}_\infty \in\mathbb{U}$, while (\ref{statement:analy-G}) only requires the convergence to hold for the \emph{specific} converging sequence $(\Qkn)_n$. Since $(\Qkn)_n$ is generated by the algorithm, it may have special characteristics regarding converging rate, distributions, etc. On the other hand, checking the stacking property may be easier than proving (\ref{statement:analy-G}). This is because the stacking property is independent of the generation number $k$.

Another point worth discussing is the introduction of $\mathbb{U}$ in Definition \ref{def:analy-stacking}. Of course, if we omit $\mathbb{U}$ (or equivalently let $\mathbb{U}=\mathbb{M}^\infty$), the stacking property will become ``stronger'' because if it holds, the convergence of the IPM is proved for the EA starting from \emph{any} initial population. However, in that case the condition is so restricted that the stacking property cannot be proved for many operators.

In Definition \ref{def:analy-stacking}, it is required that $\operatorname{\Gamma}_n(\mb{A}_n) \tod \operatorname{\Gamma}_\infty(\mb{A}_\infty)\in \mathbb{U}$ as $n\to \infty$. The sequence under investigation is $(\operatorname{\Gamma}_n(\mb{A}_n))_n$, which is a sequence of changing operators $(\operatorname{\Gamma}_n)_n$ on a sequence of changing inputs $(\mb{A}_n)_n$. As both the operators and the inputs change, the convergence of $(\operatorname{\Gamma}_n(\mb{A}_n))_n$ may still be difficult to prove. Therefore, in the following, we further derive two sufficient conditions for the stacking property.

First, let $\mb{B}_{\alpha,\beta}=\operatorname{\Gamma}_\beta(\mb{A}_\alpha)$, where $\alpha, \beta \in \mathbb{N}_+ \cup \{\infty \}$. Then, we have the following sufficient conditions for the stacking property.

\begin{theorem}[Sufficient condition 2]
	\label{thm:analy-sufficient2}
	For a space $\mathbb{U}$ and all converging sequences $\mb{A}_n \tod \mb{A}_\infty\in \mathbb{U}$, if the following two conditions
	\begin{enumerate}
		\item $\exists M\in \mathbb{N}_+$, such that for all $m>M$, $\mb{B}_{n,m}\tod \mb{B}_{\infty,m}$ uniformly as $n\to \infty$, i.e. $\sup\limits_{m>M} \distd(\mb{B}_{n,m},\mb{B}_{\infty,m}) \to 0$ as $n\to \infty$, \label{thm:analy-sufficient2-condi1}
		\item $\mb{B}_{\infty,m}\tod \mb{B}_{\infty,\infty} \in \mathbb{U}$ as $m\to \infty$, \label{thm:analy-sufficient2-condi2} 
	\end{enumerate}
	are both met, then $\operatorname{\Gamma}_\infty$ has the stacking property on $\mathbb{U}$.
\end{theorem}

\begin{theorem}[Sufficient condition 3]
	\label{thm:analy-sufficient3}
	For a space $\mathbb{U}$ and all converging sequences $\mb{A}_n \tod \mb{A}_\infty\in \mathbb{U}$, if the following two conditions
	\begin{enumerate}
		\item $\exists N\in \mathbb{N}_+$, such that for all $n>N$, $\mb{B}_{n,m}\tod \mb{B}_{n,\infty}$ uniformly as $m\to \infty$, i.e. $\sup\limits_{n>N} \distd(\mb{B}_{n,m},\mb{B}_{n,\infty}) \to 0$ as $m\to \infty$, \label{thm:analy-sufficient3-condi1}
		\item $\mb{B}_{n,\infty}\tod \mb{B}_{\infty,\infty} \in \mathbb{U}$ as $n\to \infty$, \label{thm:analy-sufficient3-condi2} 
	\end{enumerate}
	are both met, then $\operatorname{\Gamma}_\infty$ has the stacking property on $\mathbb{U}$.
\end{theorem}

In the following, we prove Theorem \ref{thm:analy-sufficient2}. Since Theorem \ref{thm:analy-sufficient2} and Theorem \ref{thm:analy-sufficient3} are symmetric in $m$ and $n$, proving one of them leads to the other. Recall that $\rho_d$ is the Prokhorov metric (Definition \ref{def:prokhorov}) and $\vee$ gets the maximal in the expression.

\begin{proof}
	$\forall \epsilon>0$, by condition \ref{thm:analy-sufficient2-condi1} in Theorem \ref{thm:analy-sufficient2}, $\exists N$ s.t. $\sup\limits_{m>M} \distd(\mb{B}_{n,m},\mb{B}_{\infty,m})<\frac{1}{2}\epsilon$ for all $n > N$. By condition \ref{thm:analy-sufficient2-condi2} in Theorem \ref{thm:analy-sufficient2}, $\exists \widetilde{M}$ s.t. $\distd(\mb{B}_{\infty,m},\mb{B}_{\infty,\infty})<\frac{1}{2}\epsilon$ for all $m > \widetilde{M}$. Now for all $l > M\vee N\vee \widetilde{M}$,
	\begin{align*}
	\distd(\mb{B}_{l,l},\mb{B}_{\infty,\infty})&\leq \distd(\mb{B}_{l,l},\mb{B}_{\infty,l})+\distd(\mb{B}_{\infty,l},\mb{B}_{\infty,\infty})\\
	&\leq\frac{1}{2}\epsilon+\frac{1}{2}\epsilon=\epsilon.
	\end{align*}
	Therefore, $\mb{B}_{n,n} \tod \mb{B}_{\infty,\infty}$ as $n\to \infty$.
\end{proof}

To understand these two theorems, consider the relationships between $\mb{A}_\alpha$ and $\mb{B}_{\alpha,\beta}$ illustrated by Fig. \ref{fig:analy-sufficient}. In the figure, the solid arrow represents the premise in Definition \ref{def:analy-stacking}, i.e. $\mb{A}_n\tod \mb{A}_\infty \in \mathbb{U}$ as $n\to \infty$. The double line arrow represents the direction to be proved for the stacking property on $\mathbb{U}$, i.e. $\mb{B}_{n,n}\tod \mb{B}_{\infty,\infty} \in \mathbb{U}$ as $n\to \infty$. The dashed arrows are the directions to be checked for Theorem \ref{thm:analy-sufficient2} to hold. The wavy arrows are the directions to be checked for Theorem \ref{thm:analy-sufficient3} to hold.

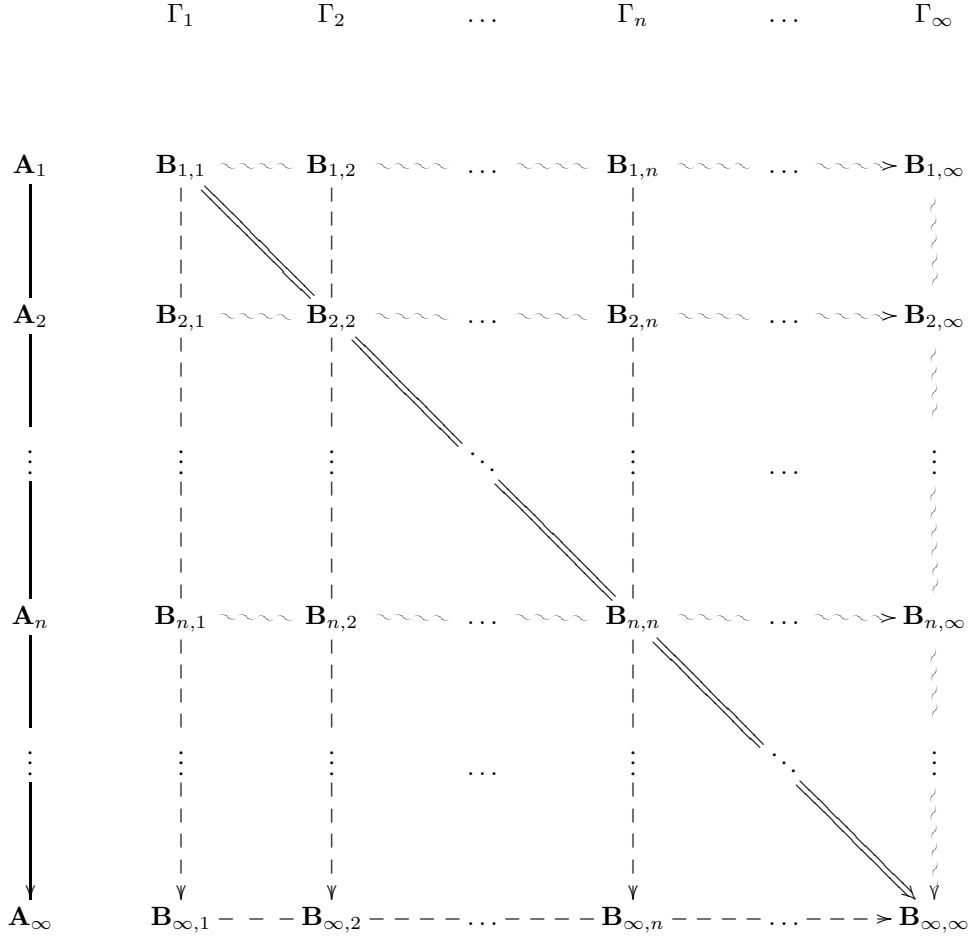
\begin{figure*}[ht]
	\centering
	\[
	\xymatrix@!
	{
		&	\operatorname{\Gamma}_1 &	\operatorname{\Gamma}_2	&\dots&\operatorname{\Gamma}_n&	\dots	&	\operatorname{\Gamma}_\infty\\
		\mb{A}_1 \ar@{-}[d]	&	\mb{B}_{1,1} \ar@{~}[r] \ar@{--}[d] \ar@{=}[rd]	&	\mb{B}_{1,2} \ar@{~}[r]\ar@{--}[d]	&\dots\ar@{~}[r]&\mb{B}_{1,n} \ar@{--}[d]\ar@{~}[r]	&	\dots\ar@{~>}[r]	&	\mb{B}_{1,\infty}\ar@{~}[d]	\\
		\mb{A}_2 \ar@{-}[d]	&	\mb{B}_{2,1} \ar@{--}[d]\ar@{~}[r]	&	\mb{B}_{2,2} \ar@{~}[r]\ar@{--}[d] \ar@{=}[rd]	&\dots\ar@{~}[r]&\mb{B}_{2,n} \ar@{--}[d]\ar@{~}[r]	&\dots	\ar@{~>}[r]	&\mb{B}_{2,\infty}\ar@{~}[d]		\\
		\vdots \ar@{-}[d]	&	\vdots \ar@{--}[d]	&	\vdots \ar@{--}[d]	&\ddots \ar@{=}[rd]	&\vdots \ar@{--}[d]	&	\dots&	\vdots\ar@{~}[d]\\
		\mb{A}_n \ar@{-}[d]	&	\mb{B}_{n,1} \ar@{~}[r]\ar@{--}[d]	&	\mb{B}_{n,2}\ar@{~}[r] \ar@{--}[d]&\dots\ar@{~}[r]&\mb{B}_{n,n}\ar@{~}[r] \ar@{--}[d]	 \ar@{=}[rd]	&	\dots\ar@{~>}[r]	&	\mb{B}_{n,\infty}\ar@{~}[d]	\\
		\vdots \ar@{->}[d]	&	\vdots \ar@{-->}[d]	&	\vdots \ar@{-->}[d]	&\dots&\vdots \ar@{-->}[d]	&	\ddots \ar@{=>}[rd]	&	\vdots\ar@{~>}[d]\\	
		\mb{A}_\infty&	\mb{B}_{\infty,1} \ar@{--}[r]	&	\mb{B}_{\infty,2} \ar@{--}[r]	&\dots \ar@{--}[r]	&\mb{B}_{\infty,n} \ar@{--}[r]	&	\dots \ar@{-->}[r]	&	\mb{B}_{\infty,\infty}	\\
	}
	\]
	\caption{Relationships between $\mb{A}_\alpha$ and $\mb{B}_{\alpha,\beta}$}
	\label{fig:analy-sufficient}
\end{figure*}

Now it is clear that Theorem \ref{thm:analy-sufficient2} and Theorem \ref{thm:analy-sufficient3} bring benefits. For example, for Theorem \ref{thm:analy-sufficient2}, instead of proving the convergence for a sequence generated by changing operators and inputs ($\mb{B}_{n,n}\tod \mb{B}_{\infty,\infty}$), this sufficient condition considers the convergence of sequences generated by the \emph{same} operator on changing inputs ($\mb{B}_{n,m}\tod \mb{B}_{\infty,m}$) and of the sequence generated by changing operators on the \emph{same} input ($\mb{B}_{\infty,m}\tod \mb{B}_{\infty,\infty}$).

The reason we introduce $M$ and $N$ in Theorem \ref{thm:analy-sufficient2} and Theorem \ref{thm:analy-sufficient3} respectively is to exclude some of the starting columns and rows in Fig. \ref{fig:analy-sufficient}, if necessary. This is useful in proving the convergence of the IPM of the $k$-ary recombination operator.

\subsection{The I.I.D. Assumption}
\label{chapter:analy-iid}

In this section, we address the issue of how to construct IPM. This issue also corresponds to how to choose the space $\mathbb{U}$ for the stacking property.

Before introducing the i.i.d. assumption, let us give an example. Consider the space $\mathbb{U}=\{\mb{x}\in \mathbb{M}^\infty|\Prob[\mb{x}=(c,c,\dots)]=1 \textnormal{ for some } c\in \mathbb{S}\}$. If the initial population follows some distribution from $\mathbb{U}$, then the population consists of all identical individuals. If an EA with proportionate selection and crossover operates on this initial population, then all subsequent populations stay the same as the initial population. An IPM of this EA can be easily constructed, and it can be easily proved that the stacking property holds as long as the EA chooses its initial population from $\mathbb{U}$. However, this is not a very interesting case. This is because $\mathbb{U}$ is too small to model real EAs.

On the other hand, if $\mathbb{U}=\{\mb{x}\in \mathbb{M}^\infty|\mb{x} \textnormal{ is exchangeable}\}$, $\mathbb{U}$ may be too big to derive meaningful results. This can be seen from our analysis in Section \ref{qiwrong} which shows that under exchangeability it is not possible to derive transition equations of marginal distributions for the simple EA.

Therefore, choosing $\mathbb{U}$ should strike a balance between the capacity and the complexity of the IPM. In the following analysis, we choose $\mathbb{U}$ to be $\UI=\{\mb{x}\in \mathbb{M}^\infty|\mb{x} \textnormal{ is i.i.d.} \}$. IPMs of EAs are constructed using the i.i.d. assumption, and we prove the convergence of the overall IPM by proving that the operators in the IPM have the stacking property on $\UI$.

We choose $\UI$ for the following reasons. First, in the real world, many EAs generate i.i.d. initial populations. Therefore this assumption is realistic. Secondly, i.i.d. random elements have the same marginal distributions. Therefore IPM can be described by transition equations of marginal distributions. Finally, there are abundant literature on the converging laws and limit theorems of i.i.d. sequences. Therefore, the difficulty in constructing IPM can be greatly reduced compared with using other modeling assumptions.

In the following, we show how to construct IPM under the i.i.d. assumption. This process also relates to condition \ref{thm:analy-sufficient2-condi2} in Theorem \ref{thm:analy-sufficient2}. It essentially describes how the IPM generates new populations.

Let the operator in the EA be $\operatorname{\Gamma}$, and the corresponding operator in $\EA{m}$ be $\operatorname{\Gamma}_m:\mathbb{M}^\infty \to \mathbb{M}^\infty$. Recall that in our framework we only study EAs consisting of c.i.i.d. operators, therefore $\operatorname{\Gamma}_m$ generates c.i.i.d. outputs by using the first $m$ elements of its input. The process that $\operatorname{\Gamma}_m$ generates \emph{each} output can be described by the conditional p.d.f. $f_{\operatorname{\Gamma}_m}(x|y_1,y_2,\dots,y_m)$. Let $\mb{a}=(\mb{a}_i)_{i=1}^\infty \in \mathbb{M}^\infty$ be the input and $\mb{b}=(\mb{b}_i)_{i=1}^\infty=\operatorname{\Gamma}_m(\mb{a})$ be the output, then the distribution of $\mb{b}$ can be completely described by its \emph{finite-dimensional} p.d.f.s
\begin{align}
\label{eqn:analy-GammaM}
&f_{\pi_l(\mb{b})}(x_1,\dots,x_l)=\nonumber\\
&\iint\limits_{\mathbb{S}^m} \! \prod\limits_{i=1}^{l}f_{\operatorname{\Gamma}_m}(x_i|y_1,\dots,y_m) \cdot f_{\pi_m(\mb{a})} (y_1,\dots,y_m)\,\mathrm{d}y_1\dots \mathrm{d}y_m
\end{align}
for every $l\in \mathbb{N}_+$. 

To derive the IPM $\operatorname{\Gamma}_\infty$ for $\operatorname{\Gamma}$, consider the case when $l=1$ and $\mb{a}\in \UI$ in (\ref{eqn:analy-GammaM}). Noting that in this case $f_{\pi_m(\mb{a})} (y_1,\dots,y_m) = \prod\limits_{i=1}^{m} f_{\mb{a}_1}(y_i)$, we have 
\begin{equation}
\label{eqn:analy-iidGammaM}
f_{\mb{b}_1}(x)=\iint\limits_{\mathbb{S}^m} \! f_{\operatorname{\Gamma}_m}(x|y_1,\dots,y_m) \cdot\prod\limits_{i=1}^{m} f_{\mb{a}_1}(y_i) \, \mathrm{d}y_1\dots \mathrm{d}y_m.
\end{equation}

Now taking $m\to \infty$, (\ref{eqn:analy-iidGammaM}) in the limit becomes the transition equation describing how $\operatorname{\Gamma}_\infty$ generates \emph{each} new individual. Let the transition equation be
\begin{equation}
\label{eqn:analy-trans}
	f_{\mb{b}_1}=\operatorname{T}_\Gamma [f_{\mb{a}_1}],
\end{equation}
and let $\mb{c}=(\mb{c}_i)_{i=1}^\infty=\operatorname{\Gamma}_\infty(\mb{a})$. Then how $\operatorname{\Gamma}_\infty$ generates $l$ individuals can be described by the finite-dimensional p.d.f.s of $\mb{c}$:
\begin{equation}
\label{eqn:analy-GammaInf}
f_{\pi_l(\mb{c})}(x_1,\dots,x_l)=\prod\limits_{i=1}^{l}\operatorname{T}_\Gamma [f_{\mb{a}_1}](x_i)
\end{equation}
for every $l\in \mathbb{N}_+$. Overall, (\ref{eqn:analy-GammaInf}) describes the mapping $\operatorname{\Gamma}_\infty:\UI \to \UI$.

To better understand the construction, it is important to realize that for $\operatorname{\Gamma}_\infty$ \emph{both} the input and the output are i.i.d. In other words, $\operatorname{\Gamma}_\infty$ generates i.i.d. population dynamics to simulate the real population dynamics produced by $\Gamma$, only that the transition equation in $\operatorname{\Gamma}_\infty$ is derived by mimicking how $\operatorname{\Gamma}$ generates each new individual on i.i.d. inputs and taking the population size to infinity. In fact, if the stacking property on $\UI$ is proved and the initial population is i.i.d., $\operatorname{\Gamma}_\infty$ will always take i.i.d. inputs and produce i.i.d. outputs. The behaviors of $\operatorname{\Gamma}_\infty$ on $\UI$ is well-defined. On the other hand, $\operatorname{\Gamma}_\infty(\mb{A}\notin \UI)$ is not defined in the construction. This leaves us freedom. We can define $\operatorname{\Gamma}_\infty(\mb{A}\notin \UI)$ freely to facilitate proving the stacking property of $\operatorname{\Gamma}_\infty$. In particular, $\mb{B}_{n, \infty}$ for $n\in \mathbb{N}_+$ in Fig. \ref{fig:analy-sufficient} can be defined freely to facilitate the analysis. 

In fact, under the i.i.d. assumption, deriving the transition equation for most operators is the easy part. The more difficult part is to prove the stacking property of $\operatorname{\Gamma}_\infty$ on $\UI$. To give an example, consider the transition equation (\ref{eqn:Qi-trans}) constructed by Qi et al. in \cite{Qi1}, which models the joint effects of proportionate selection and mutation. As our analysis in Section \ref{qiwrong} shows, it does not hold under the assumption of exchangeability. However, if the modeling assumption is i.i.d., the transition equation can be immediately proved (see our analysis in Section \ref{qiwrong}). This also applies to the transition equation built by the same authors for the uniform crossover operator (in Theorem 1 of \cite{Qi2}), where the transition equation is in fact constructed under the i.i.d. assumption. Therefore, in the following analyses, we do not refer to the explicit form of the transition equation, unless it is needed. We only assume that the transition equation is successfully constructed, and it has the form (\ref{eqn:analy-trans}) which is derived from (\ref{eqn:analy-iidGammaM}) as $m\to \infty$.

The construction of the IPM also relates partly to condition \ref{thm:analy-sufficient2-condi2} in Theorem \ref{thm:analy-sufficient2}. Comparing with this condition, it can be seen that for a successfully constructed $\operatorname{\Gamma}_\infty$, the following two facts are proved in the construction.
\begin{enumerate}
	\item $\mb{B}_{\infty,m}\tompw \mb{B}_{\infty,\infty}$ as $m\to \infty$.
	\item $\mb{B}_{\infty,\infty}\in \UI$.
\end{enumerate}
Of course, these two facts are not sufficient for this condition to hold. One still needs to prove $\mb{B}_{\infty,m}\tod \mb{B}_{\infty,\infty}$ as $m\to \infty$. In other words, one has to consider convergence of finite dimensional distributions.

Finally, we sometimes use $x$ for $x_1,\dots,x_l$ if $l$ is clear in the context. For example (\ref{eqn:analy-GammaM}) can be rewritten as \[f_{\pi_l(\mb{b})}(x)=\iint\limits_{\mathbb{S}^m} \! \prod\limits_{i=1}^{l} f_{\operatorname{\Gamma}_m}(x_i|y) \cdot f_{\pi_m(\mb{a})} (y)\, \mathrm{d}y.\]

\subsection{Analysis of the Mutation Operator}
\label{chapter:analy-mut}

Having derived sufficient conditions for the stacking property and constructed the IPM, we prove the convergence of the IPM of the mutation operator first. Mutation adds an i.i.d. random vector to each individual in the population. If the current population is $\mb{A}\in \mathbb{M}^\infty$, then the population after mutation satisfies $\Law[\mb{B}=\operatorname{\Gamma}_m(\mb{A})]=\Law(\mb{A}+\mb{X})$ for all $m\in \mathbb{N}_+$, where $\mb{X} \in \UI$ is a random element decided by the mutation operator. As the content of the mutation operator does not depend on $m$, we just write $\operatorname{\Gamma}$ to represent $\operatorname{\Gamma}_m$. To give an example, $\mb{X}$ may be the sequence $(\mb{x}_1,\mb{x}_2,\dots)$ with all $\mb{x}_i \in \mathbb{M}$ mutually independent and $\mb{x}_i \sim N(0,\identity_d)$ for all $i\in \mathbb{N}_+$, where $N(a,B)$ is the multivariate normal distribution with mean $a$ and covariance matrix $B$, and $\identity_d$ is the $d$-dimensional identity matrix. Note that every time $\operatorname{\Gamma}$ is invoked, it generates perturbations independently. For example, let $\mb{A}_1$ and $\mb{A}_2$ be two populations, then we can write $\operatorname{\Gamma}(\mb{A}_i)=\mb{A}_i + \mb{X}_i$ for $i=1,2$ satisfying $\Law(\mb{X}_1)=\Law(\mb{X}_2)=\Law(\mb{X})$ and $\{\mb{X}_i\}_{i=1,2}$ are mutually independent and independent from $\{\mb{A}_i\}_{i=1,2}$.

Next, consider $\operatorname{\Gamma}_\infty$. Recall that as an IPM, $\operatorname{\Gamma}_\infty$ simulates real population dynamics by taking i.i.d. inputs and producing i.i.d. outputs. If the marginal p.d.f.s of $\mb{A}$ and $\mb{X}$ are $f_{\mb{a}}$ and $f_{\mb{x}}$, respectively, then $\operatorname{\Gamma}_\infty(\mb{A})$ generates i.i.d. individuals whose p.d.f.s are $f_{\mb{a}}*f_{\mb{x}}$, where $*$ stands for convolution. Given the construction, we can prove the stacking property of $\operatorname{\Gamma}_\infty$.

\begin{theorem}[Mutation]
	\label{thm:analy-mut}
	Let $\operatorname{\Gamma}$ be the mutation operator, and $\operatorname{\Gamma}_\infty$ be the corresponding operator in the IPM constructed under the i.i.d. assumption, then $\operatorname{\Gamma}_\infty$ has the stacking property on $\UI$.
\end{theorem}

\begin{proof}
	We use the notations and premises in Theorem \ref{thm:analy-sufficient2}. Refer to Fig. \ref{fig:analy-sufficient}. In particular, the sequence $(\mb{A}_n)$ and the limit $\mb{A}_\infty$ are given and $\mb{A}_n \tod \mb{A}_\infty\in \UI$ as $n\to \infty$.
	
	Apparently,
	\begin{align*}
	[\operatorname{\Gamma}_m(\mb{A}_\infty)]_{m=1}^\infty &=[\operatorname{\Gamma}(\mb{A}_\infty),\operatorname{\Gamma}(\mb{A}_\infty),\dots]\\
	&\tod \operatorname{\Gamma}(\mb{A}_\infty)=\operatorname{\Gamma}_\infty(\mb{A}_\infty)\in \UI.
	\end{align*}
	Therefore, condition \ref{thm:analy-sufficient2-condi2} in Theorem \ref{thm:analy-sufficient2} is satisfied.
	
	Noting that condition \ref{thm:analy-sufficient2-condi1} in Theorem \ref{thm:analy-sufficient2} is equivalent to $\operatorname{\Gamma}(\mb{A}_n)\tod \operatorname{\Gamma}(\mb{A}_\infty)$, we prove this condition by proving that $\pi_i[\operatorname{\Gamma}(\mb{A}_n)]\tod \pi_i[\operatorname{\Gamma}(\mb{A}_\infty)]$ for all $i\in \mathbb{N}_+$. Then by Theorem \ref{thm:framework-finiteconv},  condition \ref{thm:analy-sufficient2-condi1} in Theorem \ref{thm:analy-sufficient2} holds. Then, as both conditions in Theorem \ref{thm:analy-sufficient2} are satisfied, this theorem is proved.
	
	Now, we prove $\pi_i[\operatorname{\Gamma}(\mb{A}_n)]\tod \pi_i[\operatorname{\Gamma}(\mb{A}_\infty)]$ for all $i\in \mathbb{N}_+$. First, note that $\operatorname{\Gamma}(\mb{A}_\alpha)=\mb{A}_\alpha+\mb{X}_\alpha$ for all $\alpha\in \mathbb{N}\cup \{\infty\}$. $\{\mb{X}_\alpha\in \mathbb{M}^\infty\}$ are i.i.d. and independent from $\{\mb{A}_\alpha\in \mathbb{M}^\infty\}$. In addition, for every $\alpha$, $\Law(\mb{X}_\alpha)=\Law(\mb{X})$.
	
	Since $\Law(\mb{X}_\alpha)=\Law(\mb{X})$, it is apparent that $\mb{X}_n \tod \mb{X}_\infty$. Then by Theorem \ref{thm:framework-finiteconv}, we have $\pi_i(\mb{X}_n) \tod \pi_i(\mb{X}_\infty)$ and $\pi_i(\mb{A}_n) \tod \pi_i(\mb{A}_\infty)$.
	
	Consider the product space $\mathbb{S}^i \times \mathbb{S}^i$. It is both separable and complete. Since $\pi_i(\mb{A}_\alpha)$ and $\pi_i(\mb{X}_\alpha)$ are independent, by Theorem \ref{thm:framework-product}, it follows that
	\begin{equation}
			\begin{bmatrix}
			\pi_i(\mb{A}_n)\\
			\pi_i(\mb{X}_n)
			\end{bmatrix}
			\tod
			\begin{bmatrix}
			\pi_i(\mb{A}_\infty)\\
			\pi_i(\mb{X}_\infty)
			\end{bmatrix}. \label{eqn:analy-convpi}
	\end{equation}
	
	Note that 
	\begin{equation}
	\pi_i[\operatorname{\Gamma}(\mb{A}_\alpha)]=
	\begin{bmatrix}
	\identity & \identity
	\end{bmatrix}
	\begin{bmatrix}
	\pi_i(\mb{A}_\alpha)\\
	\pi_i(\mb{X}_\alpha)
	\end{bmatrix}
	=h(	\begin{bmatrix}
	\pi_i(\mb{A}_\alpha)\\
	\pi_i(\mb{X}_\alpha)
	\end{bmatrix}
	), \label{eqn:analy-convh}	
	\end{equation}
	where $\identity$ is the identity matrix of appropriate dimension and $h:\mathbb{S}^i \times \mathbb{S}^i\to \mathbb{S}^i$ is a function satisfying $h(
	\begin{bmatrix}
	x\\
	y
	\end{bmatrix}
	)=
	\begin{bmatrix}
	I & I
	\end{bmatrix}
	\begin{bmatrix}
	x\\
	y
	\end{bmatrix}$. Apparently $h$ is continuous. Then by (\ref{eqn:analy-convpi}), (\ref{eqn:analy-convh}) and Theorem \ref{thm:framework-contimapping}, $\pi_i[\operatorname{\Gamma}(\mb{A}_n)]\tod \pi_i[\operatorname{\Gamma}(\mb{A}_\infty)]$ for any $i\in \mathbb{N}_+$.
\end{proof}

In the proof, we concatenate the input ($\mb{A}_n$) and the randomness ($\mb{X}_n$) of the mutation operator in a common product space, and represent $\operatorname{\Gamma}$ as a continuous function in that space. This technique is also used when analyzing other operators. 

\subsection{Analysis of $k$-ary Recombination}
\label{chapter:analy-recomb}

Consider the $k$-ary recombination operator and denote it by $\operatorname{\Gamma}$. In $\EA{m}$, the operator is denoted by $\operatorname{\Gamma}_m$. $\operatorname{\Gamma}_m$ works as follows. To generate a new individual, it first samples $k$ individuals from the current $m$-sized population randomly \emph{with} replacement. Assume the current population consists of $\{\mb{x}_i\}_{x=1}^m$, and the selected $k$ parents are $\{\mb{y}_i\}_{i=1}^k$, then $\{\mb{y}_i\}_{i=1}^k$ follows the probability:
\begin{equation}
\label{eqn:analy-parent}
	\Prob(\mb{y}_i=\mb{x}_j)=\frac{1}{m} \textnormal{ for all } i\in \{1,\dots,k\},j\in\{1,\dots,m\}.
\end{equation}

After the $k$ parents are selected, $\operatorname{\Gamma}_m$ produces a new individual $\mb{x}$ following the formula
\begin{equation}
\label{eqn:analy-recomb}
\mb{x}=\sum_{i=1}^{k} \mb{U}_i\mb{y}_i,
\end{equation} 
where $\{\mb{U}_i\}_{i=1}^k$ are random elements of $\mathbb{R}^{d\times d}$ (recall that $\mb{x}$ and $\mb{y}_i$ are random elements of $\mathbb{S}=\mathbb{R}^d$ modeling individuals in our framework). $\{\mb{U}_i\}_{i=1}^k$ are also independent of $\{\mb{y}_i\}_i$, and the \emph{joint} distribution of $(\mb{U}_i)_{i}$ is decided by the inner mechanism of $\operatorname{\Gamma}$. Overall, $\operatorname{\Gamma}_m$ generates the next population by repeatedly using this procedure to generate new individuals independently.

Our formulation seems strange at first sight, but it covers many real world recombination operators. For example, consider $k=2$ and $\mb{U}_1=\mb{U}_2=\frac{1}{2}\identity$. This operator is the crossover operator taking the mean of its two parents. On the other hand, if $k=2$ and the distributions of $\mb{U}_1$ and $\mb{U}_2$ satisfy
\[
\left\{
\begin{matrix}
\mb{U}_1=\operatorname{Diag}(\mb{s}_1,\mb{s}_2,\dots,\mb{s}_d)\\
\mb{U}_2=\identity-\mb{U}_1
\end{matrix}\right.
,\]
where $\operatorname{Diag}$ constructs a diagonal matrix from its inputs, $\{\mb{s}_i\}$ are i.i.d. random variables taking values in $\{0,1\}$ satisfying $\Prob(\mb{s}_i=0)=\Prob(\mb{s}_i=1)=1/2$, then this operator is the uniform crossover operator which sets value at each position from the two parents with probability $\frac{1}{2}$.

Consider the IPM $\operatorname{\Gamma}_\infty$. As stated in Section \ref{chapter:analy-iid}, we do not give the explicit form of the transition equation in $\operatorname{\Gamma}_\infty$. We assume that the IPM is successfully constructed, and the transition equation is derived by taking $m\to \infty$ in (\ref{eqn:analy-iidGammaM}). The reason for this approach is not only because deriving the transition equation is generally easier than proving the convergence of the IPM, but also the formulation in (\ref{eqn:analy-parent}) and (\ref{eqn:analy-recomb}) encompasses many real world $k$-ary recombination operators. We do not delve into details of the mechanisms of these operators and derive a transition equation for each one of them. Instead, our approach is general in that as long as the IPM is successfully constructed, our analysis on the convergence of the IPM can always be applied.

The following theorem is the primary result of our analysis for the $k$-ary recombination operator.

\begin{theorem}[$k$-ary recombination]
	\label{thm:analy-recomb}
	Let $\operatorname{\Gamma}$ be the $k$-ary recombination operator, and $\operatorname{\Gamma}_\infty$ be the corresponding operator in the IPM constructed under the i.i.d. assumption, then $\operatorname{\Gamma}_\infty$ has the stacking property on $\UI$.
\end{theorem}

\begin{proof}
	
		We use the notations and premises in Theorem \ref{thm:analy-sufficient3}. Refer to Fig. \ref{fig:analy-sufficient}. In particular, the sequence $(\mb{A}_n)$ and the limit $\mb{A}_\infty$ are given and $\mb{A}_n \tod \mb{A}_\infty\in \UI$ as $n\to \infty$.

		We prove that
		\begin{equation}
		\label{condi:analy-recomb-condi1}
			\pi_i(\mb{B}_{n,n})\tod \pi_i(\mb{B}_{\infty,\infty})
		\end{equation}
		 as $n\to \infty$ for any $i\in \mathbb{N}_+$. Then by Theorem \ref{thm:framework-finiteconv}, the conclusion follows.
		
		The overall idea to prove (\ref{condi:analy-recomb-condi1}) is that we first prove the convergence in distribution for the $k\cdot i$ selected parents, then because the recombination operator is continuous, (\ref{condi:analy-recomb-condi1}) follows. 
				
		First, we decompose the operator $\pi_i \circ \operatorname{\Gamma}_m: \mathbb{M}^\infty \to \mathbb{M}^i$. $\pi_i \circ \operatorname{\Gamma}_m$ generates the $i$ c.i.i.d. outputs one by one. This generation process can also be viewed as first selecting the $i$ groups of $k$ parents at once from the first $m$ elements of the input (in total the intermediate output is $k\cdot i$ parents not necessarily distinct), then producing the $i$ outputs one by one by using each group of $k$ parents. In the following, we describe this process mathematically.
		
		Consider $\operatorname{\Phi}_m: \mathbb{M}^\infty \to \mathbb{M}^{k\cdot i}$. Let $\mb{x}=(\mb{x}_j)_{j=1}^\infty\in \mathbb{M}^{\infty}$ and $\mb{y}=(\mb{y}_j)_{j=1}^{k\cdot i}=\operatorname{\Phi}_m(\mb{x})$. Let $\operatorname{\Phi}_m$ be described by the probability
		\begin{equation}
			\Prob(\mb{y}_j=\mb{x}_l)=\frac{1}{m} \textnormal{ for all } j\in \{1,\dots,k\cdot i\} \textnormal{ and } l\in \{1,\dots,m\}.
		\end{equation}
		In essence, $\operatorname{\Phi}_m$ describes how to select the $k\cdot i$ parents from $\mb{x}$.
		
		Consider $\operatorname{\Psi}:\mathbb{M}^{k\cdot i} \to \mathbb{M}^i$. Let
		\begin{align*}
			\mb{u}=(&\mb{u}_{1,1},\mb{u}_{1,2},\dots,\mb{u}_{1,k},\mb{u}_{2,1},\mb{u}_{2,2},\dots,\mb{u}_{2,k},\dots\dots,\\
			&\mb{u}_{i,1},\mb{u}_{i,1},\dots,\mb{u}_{i,k})\in \mathbb{M}^{k\cdot i}.
		\end{align*}
		Let $\mb{v}=(\mb{v}_j)_{j=1}^{i}=\operatorname{\Psi}(\mb{u})$. Let $\operatorname{\Psi}$ be described by
		\begin{equation}
		\label{eqn:analy-Psi}
		\mb{v}_j=\sum_{l=1}^{k} \mb{U}_{j,l}\mb{u}_{j,l} \textnormal{ for all } j\in \{1,\dots,i\}
		\end{equation} 
		in which  $\Law[(\mb{U}_{j,l})_{l=1}^k]=\Law[(\mb{U}_{l})_{l=1}^k]$, where $\{\mb{U}_{l} \}$ are decided by the recombination operator $\operatorname{\Gamma}$ as in (\ref{eqn:analy-recomb}), and $(\mb{U}_{j,l})_{l=1}^k$ are independent for different $j$. In essence, $\operatorname{\Psi}$ describes how to generate the $i$ individuals from the $k\cdot i$ parents.
		
		Now it is obvious that $\pi_i \circ \operatorname{\Gamma}_m=\operatorname{\Psi} \circ \operatorname{\Phi}_m$. Therefore, 
		\begin{equation}
		\pi_i(\mb{B}_{m,\alpha})=(\pi_i \circ \operatorname{\Gamma}_m)(\mb{A}_\alpha)=(\operatorname{\Psi} \circ \operatorname{\Phi}_m)(\mb{A}_\alpha)
		\end{equation}
		for all $m\in \mathbb{N}_+$ and $\alpha\in \mathbb{N}_+\cup \{\infty\}$.
		
		Next, consider $\pi_i \circ \operatorname{\Gamma}_\infty: \mathbb{M}^\infty \to \mathbb{M}^i$. Let $\operatorname{\Phi}_\infty=\pi_{k\cdot i}$, we prove that 
		\begin{equation}
		\label{eqn:analy-equallaw}
			\Law[(\pi_i \circ \operatorname{\Gamma}_\infty)(\mb{A})]=\Law[(\operatorname{\Psi} \circ \operatorname{\Phi}_\infty)(\mb{A})],\forall \mb{A}\in \UI.
		\end{equation}
		(\ref{eqn:analy-equallaw}) is almost obvious because both operators generate i.i.d. outputs, and both marginal p.d.f.s of the outputs follow the same distribution decided by $\operatorname{\Psi}$ on $k$ i.i.d. parents from $\mb{A}$. 
		In other words, $\operatorname{\Psi} \circ \operatorname{\Phi}_\infty$ is a model of $\pi_i \circ \operatorname{\Gamma}_\infty$ on i.i.d. inputs. The outputs they generate on the same i.i.d. input follow the same distribution.		
		
		
		Since $\mb{A}_\infty\in \UI$, by (\ref{eqn:analy-equallaw}),
				\begin{equation}
				\Law[\pi_i(\mb{B}_{\infty,\infty})=(\pi_i \circ \operatorname{\Gamma}_\infty)(\mb{A}_\infty)]=\Law[(\operatorname{\Psi} \circ \operatorname{\Phi}_\infty)(\mb{A}_\infty)].
				\end{equation}
		
		Then (\ref{condi:analy-recomb-condi1}) is equivalent to
				\begin{equation}
				\label{condi:analy-recomb-condi2}
				(\operatorname{\Psi} \circ \operatorname{\Phi}_n)(\mb{A}_n)\tod (\operatorname{\Psi} \circ \operatorname{\Phi}_\infty)(\mb{A}_\infty).
				\end{equation}
				as $n\to \infty$ for any $i\in \mathbb{N}_+$.
				
		To prove (\ref{condi:analy-recomb-condi2}), we prove the following two conditions.
		\begin{enumerate}
			\item $\exists N\in \mathbb{N}_+$, such that for all $n>N$, $\operatorname{\Phi}_m(\mb{A}_{n})\tod \operatorname{\Phi}_\infty(\mb{A}_{n})$ uniformly as $m\to \infty$, i.e. $\sup\limits_{n>N} \distd[\operatorname{\Phi}_m(\mb{A}_{n}),\operatorname{\Phi}_\infty(\mb{A}_{n})] \to 0$ as $m\to \infty$. \label{condi:analy-recomb-condi3}
			\item $\operatorname{\Phi}_\infty(\mb{A}_n)\tod \operatorname{\Phi}_\infty(\mb{A}_\infty)$ as $n\to \infty$ and $\operatorname{\Phi}_\infty(\mb{A}_\infty)$ is i.i.d. \label{condi:analy-recomb-condi4}
		\end{enumerate}
		These two conditions correspond to the conditions in Theorem \ref{thm:analy-sufficient3}. Since $\operatorname{\Phi}_\alpha$ is from $\mathbb{M}^\infty$ to $\mathbb{M}^{k\cdot i}$, we cannot directly apply Theorem \ref{thm:analy-sufficient3}. However, it is easy to extend the proof of Theorem \ref{thm:analy-sufficient3} to prove that these two conditions lead to $\operatorname{\Phi}_n(\mb{A}_{n}) \tod \operatorname{\Phi}_\infty(\mb{A}_{\infty})$ as $n\to \infty$. Then, by (\ref{eqn:analy-Psi}) it is apparent that $\Psi$ is a continuous function of its input and inner randomness. By concatenating the input and the inner randomness using the same technique as that used in the proof for Theorem \ref{thm:analy-mut}, (\ref{condi:analy-recomb-condi2}) can be proved. Then this theorem is proved.
		
	In the remainder of the proof, we prove conditions \ref{condi:analy-recomb-condi3} and \ref{condi:analy-recomb-condi4}. These conditions can be understood by replacing the top line with $\operatorname{\Phi}_m$ in Fig. \ref{fig:analy-sufficient}.
	
\subsubsection*{Proof of Condition \ref{condi:analy-recomb-condi4}}

Since $\operatorname{\Phi}_\infty=\pi_{k\cdot i}:\mathbb{S}^\infty \to \mathbb{S}^{k\cdot i}$ (recall that $\pi_{k\cdot i}$ can be viewed both as a mapping from $\mathbb{S}^\infty$ to $\mathbb{S}^{k\cdot i}$ and from $\mathbb{M}^\infty$ to $\mathbb{M}^{k\cdot i}$), $\operatorname{\Phi}_\infty$ is continuous (see Example 1.2 in \cite{conv}). Since $\mb{A}_n\tod \mb{A}_\infty$, by Theorem \ref{thm:framework-contimapping}, $\operatorname{\Phi}_\infty(\mb{A}_n)\tod \operatorname{\Phi}_\infty(\mb{A}_\infty)$. Apparently, $\operatorname{\Phi}_\infty(\mb{A}_\infty)$ is i.i.d. Therefore condition \ref{condi:analy-recomb-condi4} is proved.

It is worth noting that this simple proof comes partly from our extension of $\operatorname{\Psi} \circ \operatorname{\Phi}_\infty$ to inputs $\mb{A}\notin \UI$. In fact, the only requirement for $\operatorname{\Phi}_\infty$ is (\ref{eqn:analy-equallaw}), i.e. $\operatorname{\Psi} \circ \operatorname{\Phi}_\infty$ should model $\pi_i \circ \operatorname{\Gamma}_\infty$ on \emph{i.i.d.} inputs. By defining $\operatorname{\Phi}_\infty$ to be $\pi_{k\cdot i}$, it can take non-i.i.d. inputs such as $\mb{A}_n$. Thus this condition can be proved. In Fig. \ref{fig:analy-sufficient}, this corresponds to our freedom of defining $\mb{B}_{n,\infty},n\in \mathbb{N}_+$.

\subsubsection*{Proof of Condition \ref{condi:analy-recomb-condi3}}

To prove condition \ref{condi:analy-recomb-condi3}, we first give another representation of $\operatorname{\Phi}_m(\mb{A}_{\alpha})$, where $m >k\cdot i$ and $\alpha\in \mathbb{N}_+\cup \{\infty\}$. This representation is based on the following mutually exclusive cases.
\begin{enumerate}
	\item The $k\cdot i$ parents chosen from $\mb{A}_{\alpha}$ by $\operatorname{\Phi}_m$ are distinct.
	\item There are duplicates in the $k\cdot i$ parents which are chosen from $\mb{A}_{\alpha}$ by $\operatorname{\Phi}_m$.
\end{enumerate}

Let $\mb{s}_{m,\alpha}$ be random variables taking values in $\{0,1\}$, with probability
\begin{align}
\label{eqn:analy-p1}
p(m)&=\Prob(\mb{s}_{m,\alpha}=1) \notag\\
&=\Prob(\textnormal{$\operatorname{\Phi}_m$ chooses $k\cdot i$ distint parents from $\mb{A}_\alpha$})\notag \\
&=\frac{m\cdot (m-1)\cdot \dots \cdot (m-k\cdot i+1)}{m^{k\cdot i}}.
\end{align}
Let $\mb{x}_{m,\alpha}\in \mathbb{M}^{k\cdot i}$ follow the \emph{conditional} distribution of the $k\cdot i$ parents when $\mb{s}_{m,\alpha}=1$, and $\mb{y}_{m,\alpha}\in \mathbb{M}^{k\cdot i}$ follow the \emph{conditional} distribution of the $k\cdot i$ parents when $\mb{s}_{m,\alpha}=0$, then $\operatorname{\Phi}_m(\mb{A}_{\alpha})$ can be further represented as
\begin{equation}
\label{eqn:analy-expansionPhi}
	\operatorname{\Phi}_m(\mb{A}_{\alpha})=\mb{s}_{m,\alpha} \cdot \mb{x}_{m,\alpha} + (1-\mb{s}_{m,\alpha}) \cdot \mb{y}_{m,\alpha}.
\end{equation}
For our purpose, it is not necessary to explicitly describe the distribution of $\mb{x}_{m,\alpha}$ and $\mb{y}_{m,\alpha}$. The only useful fact is that by exchangeability of $\mb{A}_\alpha$,
\begin{equation}
\label{eqn:analy-xmalpha}
\Law(\mb{x}_{m,\alpha})=\Law[\operatorname{\Phi}_\infty(\mb{A}_\alpha)].
\end{equation}
To put it another way, $\mb{x}_{m,\alpha}$ and $\operatorname{\Phi}_\infty(\mb{A}_\alpha)$ both follow the same distribution of $k\cdot i$ \emph{distinct} individuals from the current \emph{exchangeable} population $\mb{A}_\alpha$. Also note that $\{\mb{s}_{m,\alpha}\}_\alpha$ are i.i.d. random variables. They are independent of $\mb{x}_{m,\alpha}$ and $\mb{y}_{m,\alpha}$.

Now consider $\Prob[\operatorname{\Phi}_m(\mb{A}_{n})\in A]$ for any $A\in \mathcal{S}^{k\cdot i}$. By conditioning on whether the $k\cdot i$ parents are distinct, we have
\begin{align*}
&\Prob[\operatorname{\Phi}_m(\mb{A}_{n})\in A]\\
=&p(m) \cdot \Prob(\mb{x}_{m,n}\in A)+[1-p(m)] \cdot \Prob(\mb{y}_{m,n}\in A). \end{align*}
Then by (\ref{eqn:analy-xmalpha}),
\begin{align}
&\Prob[\operatorname{\Phi}_m(\mb{A}_{n})\in A]-\Prob[\operatorname{\Phi}_\infty(\mb{A}_{n})\in A] \notag\\
=&[p(m) - 1]\cdot \Prob[\operatorname{\Phi}_\infty(\mb{A}_{n})\in A]+[1-p(m)] \cdot \Prob(\mb{y}_{m,n}\in A).
\end{align}
Since $p(m)$, $\Prob[\operatorname{\Phi}_\infty(\mb{A}_{n})\in A]$ and $\Prob(\mb{y}_{m,n}\in A)$ are all less than or equal to $1$,
\begin{align*}
	&p(m) - 1\notag\\
	\leq&[p(m) - 1]\cdot \Prob(\operatorname{\Phi}_\infty(\mb{A}_{n})\in A)\notag \\
	\leq&\Prob[\operatorname{\Phi}_m(\mb{A}_{n})\in A]-\Prob[\operatorname{\Phi}_\infty(\mb{A}_{n})\in A] \notag\\
	\leq&[p(m) - 1]\cdot \Prob(\operatorname{\Phi}_\infty(\mb{A}_{n})\in A)+[1-p(m)]\notag\\
	\leq&1-p(m),
\end{align*}
i.e. $\big|\Prob[\operatorname{\Phi}_m(\mb{A}_{n})\in A]-\Prob[\operatorname{\Phi}_\infty(\mb{A}_{n})\in A]\big|\leq 
1-p(m)$ for all $A$. Taking supremum over all $A$, we have
\begin{equation}
\label{eqn:analy-tvdist}
\sup\limits_{A\in \mathcal{S}^{k\cdot i}}\big|\Prob[\operatorname{\Phi}_m(\mb{A}_{n})\in A]-\Prob[\operatorname{\Phi}_\infty(\mb{A}_{n})\in A]\big|\leq 
1-p(m)
\end{equation}

The left hand side of (\ref{eqn:analy-tvdist}) is the total variation distance between $\operatorname{\Phi}_m(\mb{A}_{n})$ and $\operatorname{\Phi}_\infty(\mb{A}_{n})$. It is an upper bound of the Prokhorov distance (see \cite{distance} for its definition and properties). Since the bound $1-p(m)$ is uniform with respect to $n$ and $p(m)\to 1$ as $m\to \infty$, we have
\begin{equation}
	\sup\limits_{n} \distd[\operatorname{\Phi}_m(\mb{A}_{n}),\operatorname{\Phi}_\infty(\mb{A}_{n})]\leq 1-p(m) \to 0 \textnormal{ as $m\to \infty$}.
\end{equation}
This is exactly condition \ref{condi:analy-recomb-condi3}. Therefore this theorem is proved.

Or, if we do not want to use the total variance distance, we have the following result for any $\operatorname{\Phi}_\infty(\mb{A}_{\infty})$-continuity set $A\in \mathbb{S}^{k\cdot i}$.

\begin{align}
	&\big|\Prob[\operatorname{\Phi}_n(\mb{A}_{n})\in A]-\Prob[\operatorname{\Phi}_\infty(\mb{A}_{\infty})\in A]\big| \notag\\
	\leq&
	\big|\Prob[\operatorname{\Phi}_n(\mb{A}_{n})\in A]-\Prob[\operatorname{\Phi}_\infty(\mb{A}_{n})\in A]\big|+\notag\\
	&\big|\Prob[\operatorname{\Phi}_\infty(\mb{A}_{n})\in A]-\Prob[\operatorname{\Phi}_\infty(\mb{A}_{\infty})\in A]\big| \notag\\
	\leq&
	1-p(n)+\big|\Prob[\operatorname{\Phi}_\infty(\mb{A}_{n})\in A]-\Prob[\operatorname{\Phi}_\infty(\mb{A}_{\infty})\in A]\big|.
	\label{eqn:analy-otherway}
\end{align}
Since we already proved $\operatorname{\Phi}_\infty(\mb{A}_{n})\tod \operatorname{\Phi}_\infty(\mb{A}_{\infty})$, by \ref{thm:framework-portmanteau-condi4}) in Theorem \ref{thm:framework-portmanteau}, $\big|\Prob[\operatorname{\Phi}_\infty(\mb{A}_{n})\in A]-\Prob[\operatorname{\Phi}_\infty(\mb{A}_{\infty})\in A]\big|\to 0$. Then apparently (\ref{eqn:analy-otherway}) converges to $0$. Noting that $A$ is arbitrary, by applying \ref{thm:framework-portmanteau-condi4}) in Theorem \ref{thm:framework-portmanteau} again, $\operatorname{\Phi}_n(\mb{A}_{n})\tod \operatorname{\Phi}_\infty(\mb{A}_{\infty})$ is proved.
\end{proof}

We give a brief discussion of the proof. In our opinion, the most critical step of our proof is decomposing the $k$-ary recombination operator to two sub-operators, one is responsible for selecting parents ($\operatorname{\Phi}$), the other is responsible for combining them ($\operatorname{\Psi}$). In addition, for parent selection, the sub-operator does \emph{not} use the information of fitness values. Rather, it selects parents ``blindly'' according to its own rules (uniform sampling with replacement). This makes the operator $\operatorname{\Phi}$ easier to analyze because the way it selects parents does not rely on its input. Therefore we can prove uniform convergence in (\ref{eqn:analy-tvdist}).

Another point worth mentioning is the choice of Theorem \ref{thm:analy-sufficient3} in our proof. Though Theorem \ref{thm:analy-sufficient2} and Theorem \ref{thm:analy-sufficient3} are symmetric, the difficulties of proving them are quite different. In fact, it is very difficult to prove the uniform convergence condition in Theorem \ref{thm:analy-sufficient2}.
 
Finally, our proof can be easily extended to cover $k$-ary recombination operators using uniform sampling \emph{without} replacement to select parents for each offspring. The overall proof framework roughly stays the same.

\subsection{Summary}
\label{chapter:analy-sum}

In this section, we analyzed the simple EA within the proposed framework. As the analysis shows, although the convergence of IPM is rigorously defined, actually proving the convergence for operators usually takes a lot of effort. We derived sufficient conditions under which the convergence of IPM is guaranteed, and discussed how IPM is constructed. Then we used various techniques to analyze the mutation operator and the $k$-ary recombination operator. It can be seen that although the sufficient conditions can provide general directions for the proofs, there are still much details to be worked out in order to analyze different operators.

To appreciate the significance of our work, it is worth noting that in \cite{Qi1,Qi2} the convergence of the IPMs of the mutation operator, the uniform crossover operator and the proportionate selection operator was not properly proved, and the issue of stacking of operators and iterating the algorithm was not addressed at all. In this paper, however, we have proved the convergence of IPMs of several general operators. Since these general operators cover the operators studied in \cite{Qi1,Qi2} as special cases, the convergence of the IPMs of mutation and uniform crossover are actually proved in this paper. Besides, our proof does not depend on the explicit form of the transition equation of the IPM. As long as the IPM is constructed under the i.i.d. assumption, our proof is valid.

As a consequence of our result, consider the explicit form of the transition equation for the uniform crossover operator derived in Section II in \cite{Qi2}. As the authors' proof was problematic and incomplete, the derivation of the transition equation was not well founded. However, it can be seen that the authors' derivation is in fact equivalent to constructing the IPM under the i.i.d. assumption. Since we have already proved the convergence of IPM of the $k$-ary crossover operator, the analysis in \cite{Qi2} regarding the explicit form of the transition equation can be retained.

\section{Conclusion and Future Research}
\label{conclusion}

In this paper, we revisited the existing literature on the theoretical foundations of IPMs, and proposed an analytical framework for IPMs based on convergence in distribution for random elements taking values in the metric space of infinite sequences. Under the framework, commonly used operators such as mutation and recombination were analyzed. Our approach and analyses are new. There are many topics worth studying for future research.

Perhaps the most immediate topic is to analyze the proportionate selection operator in our framework. The reason that the mutation operator and the $k$-ary recombination operator can be readily analyzed is partly because they do not use the information of the fitness value. Also to generate a new individual, these operators draw information from a fixed number of parents. On the other hand, to generate each new individual, the proportionate selection operator actually gathers and uses fitness values of the whole population. This makes analyzing proportionate selection difficult. In fact, we have not proven the convergence of the IPM of proportionate selection, though we have obtained the following two partial analytical results under two different metrics: the Prokhorov metric and the total variation metric.

\begin{theorem}[Analysis under the Prokhorov metric]
	\label{thm:analy-proksel}
	Let $\operatorname{\Gamma}$ be the combined operator of mutation and proportionate selection in the simple EA, and $\operatorname{\Gamma}_\infty$ be the IPM constructed under the i.i.d. assumption with the transition equation (\ref{eqn:Qi-trans}). Assume the objective function $g$ and the conditional p.d.f. for mutation $f_w(x|y)$ satisfy the two conditions in Theorem \ref{thm:Qi-conv}. For $\alpha, \beta \in \mathbb{N}_+ \cup \{\infty \}$, let $\mb{A}_\alpha$ be random elements of $\mathbb{S}^\infty$ and $\mb{A}_n\tod \mb{A}_\infty\in \UI$, and $\mb{B}_{\alpha,\beta}=\operatorname{\Gamma}_\beta(\mb{A}_\alpha)$. Then the following statements are true.
	\begin{enumerate}
		\item $\mb{B}_{n,m}\tod \mb{B}_{\infty,m}$ as $n\to \infty$.\label{thm:analy-proksel-condi1}
		\item $\mb{B}_{\infty,m} \tod \mb{B}_{\infty,\infty}\in \UI$ as $m\to \infty$.\label{thm:analy-proksel-condi2}
	\end{enumerate}
\end{theorem}

Comparing with Theorem \ref{thm:analy-sufficient2}, it can be seen that condition \ref{thm:analy-sufficient2-condi2} in Theorem \ref{thm:analy-sufficient2} is proved. The only difference is that condition \ref{thm:analy-sufficient2-condi1} in the theorem requiring the \emph{uniform} convergence of $\mb{B}_{n,m}\tod \mb{B}_{\infty,m}$ as $n\to \infty$ has not been proved yet.

Let $\totv$ stand for total variation convergence. Our analysis of proportionate selection under the total variation distance yields the following results. 

\begin{theorem}\label{thm:analy-tvcore2}
	For c.i.i.d. operators, if $\mb{A}_n\totv \mb{A}_\infty$, then $\mb{B}_{n,m}\totv \mb{B}_{\infty,m}$ uniformly with respect to $m$, i.e. $\sup\limits_{m} \disttv(\mb{B}_{n,m},\mb{B}_{\infty,m})\to 0$ as $n\to \infty$. 
\end{theorem}

\begin{theorem} \label{thm:analy-tvcore}
	For the proportionate selection operator, $\pi_l(\mb{B}_{\infty,m})\totv \pi_l(\mb{B}_{\infty,\infty})$ as $m\to \infty$ for all $l\in \mathbb{N}_+$.
\end{theorem}

\begin{theorem} \label{thm:analy-tvcore1}
	$\mb{B}_{\infty,m}\totv \mb{B}_{\infty,\infty}$ if and only $\pi_l(\mb{B}_{\infty,m})\totv \pi_l(\mb{B}_{\infty,\infty})$ uniformly with respect to $l$, i.e. $\sup\limits_{l} \disttv(\pi_l(\mb{B}_{\infty,m}),\pi_l(\mb{B}_{\infty,\infty}))\to 0$ as $m\to \infty$. 
\end{theorem}

Comparing with Theorem \ref{thm:analy-sufficient2}, Theorem \ref{thm:analy-tvcore2} proves condition \ref{thm:analy-sufficient2-condi1} requiring the column-wise \emph{uniform} convergence in Fig. \ref{fig:analy-sufficient}. Theorem \ref{thm:analy-tvcore} proves convergence of finite-dimensional distributions of the last row in Fig. \ref{fig:analy-sufficient}. However, Theorem \ref{thm:analy-tvcore1} states that condition \ref{thm:analy-sufficient2-condi2} in Theorem \ref{thm:analy-sufficient2} requires the \emph{uniform} convergence of finite-dimensional distributions of the last row. We have not proven this convergence yet.

In summary, our results show that proving the convergence of $\mb{B}_{\infty,m}\to \mb{B}_{\infty,\infty}$ is more difficult under the total variation metric than under the Prokhorov metric, while in proving the uniform convergence of $\mb{B}_{n,m}\to \mb{B}_{\infty,m}$, it is the other way around.

We think further analysis on proportionate selection can be conducted in the following two directions.
\begin{enumerate}
	\item In the analyses we tried to prove the stacking property on $\UI$ for the IPM of proportionate selection. Apart from more efforts trying to prove/disprove this property, it is worth considering modifying the space $\UI$. For example, we can incorporate the \emph{rate} of convergence into the space. If we can prove the stacking property on $\UI \cap \mathbb{U}$ where $\mathbb{U}$ is the space of converging sequences with rate $O(h(n))$, it is also a meaningful result.
	\item Another strategy is to bypass the sufficient conditions and return to Definition \ref{def:framework-convipm} to prove $\Qkn\tod \Qkinf$ for every $k$. This is the original method. In essence, it requires studying the convergence of nesting integrals.
\end{enumerate}

Apart from proportionate selection, it is also worth studying whether other operators, such as ranking selection, can be analyzed in our framework. As many of these operators do not generate c.i.i.d. offspring, it makes deriving the IPM and proving its convergence difficult, if not impossible. In this regard, we believe new techniques of modeling and extensions of the framework are fruitful directions for further research.

Finally, it is possible to extend the concept of ``incidence vectors'' proposed by Vose to the continuous search space. After all, as noted by Vose himself, incidence vectors can also be viewed as marginal p.d.f.s of individuals. As a consequence, the cases of EAs on discrete and continuous solution spaces indeed do bear some resemblance. By an easy extension, the incidence vectors in the continuous space can be defined as functions with the form $\sum c_i\delta(x_i)$, where $\delta$ is the Dirac function and $c_i$ is the rational number representing the fraction that $x_i$ appears in the population. If similar analyses based on this extension can be carried out, many results in \cite{vose1,vose2,vose3,vose4} can be extended to the continuous space.

\ifCLASSOPTIONcaptionsoff
  \newpage
\fi



%
%
%

\bibliographystyle{IEEEtran}
\bibliography{TEC_arxiv}

\begin{thebibliography}{10}
\providecommand{\url}[1]{#1}
\csname url@samestyle\endcsname
\providecommand{\newblock}{\relax}
\providecommand{\bibinfo}[2]{#2}
\providecommand{\BIBentrySTDinterwordspacing}{\spaceskip=0pt\relax}
\providecommand{\BIBentryALTinterwordstretchfactor}{4}
\providecommand{\BIBentryALTinterwordspacing}{\spaceskip=\fontdimen2\font plus
\BIBentryALTinterwordstretchfactor\fontdimen3\font minus
  \fontdimen4\font\relax}
\providecommand{\BIBforeignlanguage}[2]{{%
\expandafter\ifx\csname l@#1\endcsname\relax
\typeout{** WARNING: IEEEtran.bst: No hyphenation pattern has been}%
\typeout{** loaded for the language `#1'. Using the pattern for}%
\typeout{** the default language instead.}%
\else
\language=\csname l@#1\endcsname
\fi
#2}}
\providecommand{\BIBdecl}{\relax}
\BIBdecl

\bibitem{Qi1}
X.~Qi and F.~Palmieri, ``Theoretical analysis of evolutionary algorithms with
  an infinite population size in continuous space. {P}art {I}: Basic properties
  of selection and mutation,'' \emph{IEEE Transactions on Neural Networks},
  vol.~5, no.~1, pp. 102--119, 1994.

\bibitem{Qi2}
------, ``Theoretical analysis of evolutionary algorithms with an infinite
  population size in continuous space. {P}art {II}: Analysis of the
  diversification role of crossover,'' \emph{IEEE Transactions on Neural
  Networks}, vol.~5, no.~1, pp. 120--129, 1994.

\bibitem{Qi3}
G.~Yong, Q.~Xiaofeng, and F.~Palmieri, ``Comments on ``theoretical analysis of
  evolutionary algorithms with an infinite population size in continuous space.
  {I}. basic properties of selection and mutation'' [with reply],'' \emph{IEEE
  Transactions on Neural Networks}, vol.~9, no.~2, pp. 341--343, 1998.

\bibitem{vose1}
A.~Nix and M.~D. Vose, ``Modeling genetic algorithms with markov chains,''
  \emph{Annals of Mathematics and Artificial Intelligence}, vol.~5, no.~1, pp.
  79--88, 1992.

\bibitem{vose2}
M.~D. Vose, ``What are genetic algorithms? a mathematical prespective,'' in
  \emph{Evolutionary Algorithms}, ser. The IMA Volumes in Mathematics and its
  Applications, L.~Davis, K.~De~Jong, M.~Vose, and L.~Whitley, Eds.\hskip 1em
  plus 0.5em minus 0.4em\relax New York: Springer, 1999, vol. 111, pp.
  251--276.

\bibitem{vose3}
------, \emph{The Simple Genetic Algorithm : Foundations and Theory}.\hskip 1em
  plus 0.5em minus 0.4em\relax Cambridge, Mass.; London;: MIT Press, 1999.

\bibitem{vose4}
------, ``Infinite population {GA} tutorial,'' The University of Tennessee,
  Knoxville, Tech. Rep. ut-cs-04-533, 2004.

\bibitem{exch1}
R.~L. Taylor, P.~Z. Daffer, and R.~F. Patterson, \emph{Limit theorems for sums
  of exchangeable random variables}.\hskip 1em plus 0.5em minus 0.4em\relax
  Rowman \& Allanheld, 1985.

\bibitem{prob}
S.~C. Port, \emph{Theoretical probability for applications}.\hskip 1em plus
  0.5em minus 0.4em\relax New York: John Wiley \& Sons, 1994.

\bibitem{conv}
P.~Billingsley, \emph{Convergence of Probability Measures}, 2nd~ed.\hskip 1em
  plus 0.5em minus 0.4em\relax New York: John Wiley \& Sons, 1999.

\bibitem{foundation}
O.~Kallenberg, \emph{Foundations of modern probability}, 2nd~ed.\hskip 1em plus
  0.5em minus 0.4em\relax New York: Springer, 2002.

\bibitem{dudley}
R.~M. Dudley, \emph{Real Analysis and Probability}, 2nd~ed.\hskip 1em plus
  0.5em minus 0.4em\relax Cambridge University Press, 2002.

\bibitem{distance}
A.~L. Gibbs and F.~E. Su, ``On choosing and bounding probability metrics,''
  \emph{International Statistical Review}, vol.~70, no.~3, pp. 419--435, 2002.

\end{thebibliography}

%

\begin{IEEEbiography}{Bo Song}
Biography text here.
\end{IEEEbiography}

\begin{IEEEbiography}{Victor O.K. Li}
Biography text here.
\end{IEEEbiography}





\end{document}